\newif\ifextended
\extendedtrue

\documentclass[sigplan,screen,nonacm=true]{acmart}

\usepackage[ruled]{algorithm2e}

\usepackage[utf8]{inputenc} \usepackage[T1]{fontenc}    \usepackage{url}            \usepackage{booktabs}       \usepackage{amsfonts}       \usepackage{nicefrac}       \usepackage{microtype}      \usepackage[varqu]{zi4}
\usepackage{alltt}
\usepackage[fleqn]{mathtools}
\usepackage{bm}
\usepackage{listings}
\usepackage{subcaption}
\usepackage{url}
\usepackage{stmaryrd}
\usepackage{enumitem}
\usepackage{caption}
\usepackage{todonotes}
\usepackage{wrapfig}
\usepackage{float}
\usepackage{alltt}
\usepackage{pifont}
\newcommand{\smark}{{\color{green!60!black}\ding{51}}}
\newcommand{\emark}{{\color{red!80!black}\ding{55}}}
\newcommand{\mmark}{{\color{orange!90!black}\ding{109}}}
\usepackage{xspace}

\usepackage{tikz}
\usetikzlibrary{bayesnet}
\usetikzlibrary{arrows}
\usetikzlibrary{backgrounds}

\usepackage{cleveref}
\crefname{figure}{Figure}{Figures}
\crefname{table}{Table}{Tables}

\DeclareMathOperator*{\argmin}{argmin}
\newcommand{\matpar}{\,\Vert\,}
\newcommand{\por}{\,\vert\,}
\newtheorem{assumption}{Assumption}

\definecolor{blue}{rgb}{0.0, 0.0, 0.55}

\lstnewenvironment{lstpython}{\lstset{language=Python}}
{}
\lstnewenvironment{lstpythontable}{\lstset{language=Python,basicstyle=\fontsize{8}{9}\ttfamily,aboveskip=0em,belowskip=0em,xleftmargin=0em}}
{}
\lstnewenvironment{lstpythons}{\lstset{language=Python,basicstyle=\ttfamily\small}}
{}
\def\python{\lstinline[language=Python, basicstyle=\normalsize\ttfamily]}

\lstdefinelanguage{Stan}[]{}{
  morekeywords=[3]{guide,parameters,networks,variational},morekeywords=[2]{functions,model,data,parameters,quantities,transformed,generated},
  morekeywords=[1]{for,in,while,repeat,until,if,then,else,true,false,int,real,vector,simplex,ordered,positive_ordered,row_vector,unit_vector,matrix,cholesky_factor_corr, cholesky_factor_cov,corr_matrix, cov_matrix,target,skip,let,factor,observe,sample,return},
  morecomment={[s]{/*}{*/}},
}
\lstnewenvironment{lststan}{\lstset{language=Stan}}
{}
\lstnewenvironment{lststantable}{\lstset{language=Stan,basicstyle=\fontsize{8}{9}\ttfamily,aboveskip=0em,belowskip=0em,xleftmargin=0em}}
{}
\lstnewenvironment{lststantablesmall}{\lstset{language=Stan,basicstyle=\fontsize{7}{8}\ttfamily,aboveskip=0em,belowskip=0em,xleftmargin=0em}}
{}
\def\stan{\lstinline[language=Stan, basicstyle=\normalsize\ttfamily]}

\newcommand{\lb}{\llbracket}
\newcommand{\rb}{\rrbracket}

\newcommand{\sem}[1]{\left\lb #1 \right\rb}
\newcommand{\psem}[1]{\{\mkern-3.8mu[ #1 ]\mkern-3.8mu\}}

\newcommand{\pcomp}[1]{\mathcal{C}\!\left( #1 \right)}
\newcommand{\dcomp}[2]{\mathcal{C}_{#1}\left( #2 \right)}
\newcommand{\scomp}[2]{\mathcal{C}_{#1}\left( #2 \right)}

\newcommand{\kcomp}[2]{\mathcal{C}\left( #1, #2 \right)}

\begin{document}

\title{Compiling Stan to Generative Probabilistic Languages and Extension to Deep Probabilistic Programming}

\author{Guillaume Baudart}
\affiliation{
  \institution{INRIA Paris\\\mbox{École normale supérieure -- PSL University}}
  \country{France}
}

\author{Javier Burroni}
\affiliation{
  \institution{UMass Amherst}
  \country{USA}
}

\author{Martin Hirzel}
\affiliation{
  \institution{\mbox{MIT-IBM Watson AI Lab, IBM Research}}
  \country{USA}
}

\author{Louis Mandel}
\affiliation{
  \institution{\mbox{\hspace*{-2mm}MIT-IBM Watson AI Lab, IBM Research}}
  \country{USA}
}

\author{Avraham Shinnar}
\affiliation{
  \institution{\mbox{\hspace*{2mm}MIT-IBM Watson AI Lab, IBM Research}}
  \country{USA}
}

\renewcommand{\shortauthors}{}

\newcommand{\speedup}{2.3}
\newcommand{\nbbench}{26\xspace}

\begin{abstract}
  Stan is a probabilistic programming language that is popular in the statistics community, with a high-level syntax for expressing probabilistic models.
Stan differs by nature from generative probabilistic programming languages like Church, Anglican, or Pyro.
This paper presents a comprehensive compilation scheme to compile any Stan model to a generative language and proves its correctness.
We use our compilation scheme to build two new backends for the Stanc3 compiler targeting Pyro and NumPyro.
Experimental results show that the NumPyro backend yields a \speedup x speedup compared to Stan in geometric mean over \nbbench benchmarks.
Building on Pyro we extend Stan with support for explicit variational inference guides and deep probabilistic models.
That way, users familiar with Stan get access to new features without
having to learn a fundamentally new language.

 \end{abstract}



\maketitle

\ifextended
\newpage
\fi
\section{Introduction}\label{sec:introduction}

Probabilistic Programming Languages (PPLs) are designed to describe probabilistic models and run inference on them.
There exists a variety of PPLs.
BUGS~\cite{lunn2009bugs}, JAGS~\cite{plummer2003jags}, and Stan~\cite{carpenter2017stan} focus on efficiency, constraining what is expressible to a subset of models which support fast inference techniques.
These languages enjoy broad adoption by the statistics and social sciences communities~\cite{gelman2006data, gelman2013bayesian,carlin2008bayesian}.
\emph{Generative languages} like Church~\cite{goodman_et_al_2008}, Anglican~\cite{tolpin_et_al_2016}, WebPPL~\cite{goodman_stuhlmuller_2014}, Pyro~\cite{bingham_et_al_2019}, and Gen~\cite{cusamotowner_et_al_2019} describe \emph{generative models}, i.e., stochastic procedures that simulate the data generation process.
Coming from a core programming languages heritage, generative PPLs typically support rich control constructs and models over structured data.
Generative PPLs are increasingly used in machine-learning research and are rapidly incorporating new ideas, such as Stochastic Gradient Variational Inference~(SVI), in what is now called Deep Probabilistic Programming~\cite{bingham_et_al_2019, baudart_hirzel_mandel_2018, tran_et_al_2017}.

While the semantics of probabilistic languages have been extensively
studied~\cite{kozen_1981,gordon_et_al_2014,staton_2017,GorinovaGS19},
to the best of our knowledge little is known about the relationship
between Stan and generative PPLs.
We show that a simple 1:1 translation is incorrect or incomplete for a set of subtle but widely-used 
Stan features, such as left expressions or implicit priors.

This paper formalizes the relationship between Stan and generative PPLs and introduces, with correctness proof, a \emph{comprehensive} compilation scheme that can compile any Stan program to a generative PPL.
This enables leveraging the rich set of existing Stan models for testing,
benchmarking, or experimenting with new features or inference
techniques.
Based on this compilation scheme we implemented two new backends for the Stanc3 compiler targeting Pyro~\cite{bingham_et_al_2019} and NumPyro~\cite{numpyro_2019}, a JAX~\cite{jax2018github} based version of Pyro.
Both Pyro and NumPyro runtimes offer NUTS~\cite{homan2014nuts} (No U-Turn Sampler), an optimized Hamiltonian Monte-Carlo (HMC) algorithm that is the preferred inference method for Stan.
We can thus validate our approach against Stan.
Results show that models compiled using our NumPyro backend yield equivalent results while being \speedup x faster than their Stan counterpart in the geometric mean over \nbbench benchmarks.
Our compiler and runtime library are open-source at \url{https://github.com/deepppl}.

In addition, recent probabilistic languages offer new features to program and reason about complex models.
Our compilation scheme combined with conservative extensions of Stan can be used to make these benefits available to Stan users.
As a proof of concept, based on our Pyro backend, this paper introduces DeepStan: Stan extended with support for explicit variational guides and deep probabilistic models.
Variational inference was central in the design of Pyro and programmers can easily craft their own inference guides to run variational inference on probabilistic models.
Pyro is built on top of PyTorch~\cite{paszke_et_al_2017}. Programmers can thus seamlessly import neural networks designed with the state-of-the-art API provided by PyTorch.

\medskip
This paper makes the following contributions:
\begin{itemize}[leftmargin=*]
  \item A comprehensive compilation scheme from Stan to a generative
    PPL (\Cref{sec:overview}).
  \item Correctness proof of the compilation scheme (\Cref{sec:formal}).
\item An open-source implementation of two new backends for Stanc3 targeting Pyro and NumPyro~(\Cref{sec:implem}).
  \item An extension of Stan with explicit variational inference
    guides and deep probabilistic models~(\Cref{sec:deepstan}).
\end{itemize}

\noindent
The fundamental new result of this paper is that every Stan
program can be expressed as a generative probabilistic program.
Besides advancing the understanding of probabilistic programming
languages at a fundamental level, this paper aims to provide practical
benefits to the communities of both Stan and Pyro.
From the perspective of the Stan community, this paper presents a new
competitive compiler backend and additional capabilities while retaining
familiar syntax and semantics.
This compiler can thus be used to migrate existing Stan codebases to Pyro and NumPyro.
From the perspective of the Pyro community, this paper presents a new
compiler frontend that unlocks many existing real-world
models as examples and benchmarks.

\ifextended
This paper is a version with appendices presenting the proofs and the evaluation results of the paper published at PLDI 2021~\cite{deepstan-short}.
\else
An extended version of the paper with appendices presenting the proofs and the evaluation results is available~\cite{deepstan-extended}.
\fi

 \section{Overview}\label{sec:overview}

This section shows how to compile Stan~\cite{carpenter2017stan},
which specifies a joint probability distribution,
to a generative PPL like Church, Anglican, or Pyro.
This translation also demonstrates that Stan's expressive power is at most as large as that of generative languages, a fact that was not clear before our paper.

\begin{figure}
  \centering
  \begin{minipage}{0.58\linewidth}
    \begin{lststantable}
data {
  int N;
  int<lower=0,upper=1> x[N]; } 
parameters {
  real<lower=0,upper=1> z; } 
model { 
  z ~ beta(1, 1);
  for (i in 1:N) x[i] ~ bernoulli(z); }
    \end{lststantable}
  \end{minipage}
  \begin{minipage}{0.35\linewidth}
    \centering
    \vspace{-1em}
    \scalebox{0.85}{
    \begin{tikzpicture}

  \node[latent] (z) {$z$};
  \node[obs, right=of z] (x) {$x$};
  \plate[]{model}{
      (x)
  }{$N$};
  \factoredge {} {z} {x};
  \node[below=of z, yshift=1cm, xshift=-0.2cm] (model-caption) {$p(z \por x_1, \dots, x_N)$};
  \end{tikzpicture}}
\end{minipage}
  \vspace{-1em}
  \caption{\label{fig:coin_stan}Biased coin model in Stan.}
\end{figure}

As a running example, consider the biased coin model in
\cref{fig:coin_stan}. Stan's \stan{data} block defines observed
variables for $N$ coin flips $x_i$, $i\in [1:N]$, which can be~$0$ for
tails or~$1$ for heads. The \stan{parameters} block introduces a
latent variable~$z \in [0, 1]$ for the bias of the coin.  The
\stan{model} block sets the prior of the bias~$z$ to
$\textrm{Beta}(1,1)$, i.e., a uniform distribution over $[0, 1]$.  The
\stan{for} loop indicates that coin flips~$x_i$ are independent and
identically distributed (IID) and depend on~$z$ via a Bernoulli
distribution.  Given concrete observed coin flips, inference yields a
posterior distribution for~$z$ conditioned on $x_1,\ldots,x_N$.

\subsection{Generative translation}\label{sec:generative_translation}

Generative PPLs are general-purpose languages extended with two probabilistic constructs~\cite{van2018introduction, gordon_et_al_2014,staton_2017}:
\python{sample($D$)} generates a sample from a distribution~$D$ and \python{factor($v$)} assigns a score~$v$ to the current execution trace.
Typically, \python{factor} is used to condition the model on input data~\cite{tolpin_et_al_2016}.
We also introduce \python{observe($D$,$v$)} as a syntactic shortcut for \python{factor($D_{\rm{pdf}}$($v$))} where $D_{\rm{pdf}}$ denotes the probability density function of~$D$.
This construct penalizes executions according to the score of~\python{$v$} w.r.t.~$D$ which captures the assumption that the observed data~$v$ follows the distribution~$D$.

\begin{figure}
\begin{subfigure}[t]{0.45\columnwidth}
\begin{lstpythontable}
def model(N, x):
  z = sample(
        beta(1.,1.))
  for i in range(0, N): 
    observe(
      bernoulli(z), x[i])
  return z

\end{lstpythontable}
\caption{Generative scheme}
\label{fig:compile_coin_generative}
\end{subfigure}
\vline
\hspace{0.5em}
\begin{subfigure}[t]{0.5\columnwidth}
\begin{lstpythontable}
def model(N, x):
  z = sample(uniform(0.,1.))
  observe(beta(1.,1.), z)
  for i in range(0, N): 
    observe(
      bernoulli(z), x[i])
  return z
\end{lstpythontable}
\caption{Comprehensive scheme}
\label{fig:compile_coin_comprehensive}
\end{subfigure}
\caption{Compiled coin model of \Cref{fig:coin_stan}.}
\label{fig:compile_coin}
\end{figure}

\paragraph{Compilation.}
Stan uses the same syntax~\stan{v ~ $D$} for both observed and latent variables.
The distinction comes from the kind of the left-hand-side variable: observed variables are declared in the \stan{data} block, latent variables are declared in the \stan{parameters} block.
A straightforward \emph{generative translation} compiles a statement \stan{v ~ $D$} into \python{v = sample($D$)} if \stan{v} is a parameter or \python{observe($D$, v)} if~\stan{v} is data.
For example, \Cref{fig:compile_coin_generative} shows the compiled
(using the generative scheme) version of the Stan model of
\cref{fig:coin_stan} in Python syntax.

\subsection{Non-generative features}
\label{sec:non-gen-feat}

\begin{table*}[!t]
  \centering
  \caption{\label{tab:stan_features}Stan features that defy generative translation: prevalence, example, and compilation.}
\vspace{-1em}
{\small\begin{tabular}[t]{@{}lrll@{}}
\textsc{Feature} & \% & \textsc{Example} & \textsc{Compilation}\\\toprule
Left \mbox{expression} & 15
& \begin{minipage}{14em}\begin{lststantable}
sum(phi) ~ normal(0, 0.001*N);
\end{lststantable}\end{minipage}
& \begin{minipage}{16.5em}\begin{lstpythontable}
observe(Normal(0.,0.001*N), sum(phi))
\end{lstpythontable}\end{minipage}\\\midrule
Multiple updates & 8
& \begin{minipage}{14em}\begin{lststantable}
phi_y ~ normal(0,sigma_py);
phi_y ~ normal(0,sigma_pt)
\end{lststantable}\end{minipage}
&\begin{minipage}{16.5em}\begin{lstpythontable}
observe(Normal(0.,sigma_py), phi_y);
observe(Normal(0.,sigma_pt), phi_y)
\end{lstpythontable}\end{minipage}
\\\midrule
Implicit prior & 58
&\begin{minipage}{14em}\begin{lststantable}
real alpha0; 
/* missing 'alpha0 ~ ...' */
\end{lststantable}\end{minipage}
&\begin{minipage}{16.5em}\begin{lstpythontable}
alpha0 = sample(improper_uniform())
\end{lstpythontable}\end{minipage}\\\midrule
\end{tabular}}
\end{table*}

In Stan, a model represents the unnormalized density of the joint distribution of the parameters defined in the \stan+parameters+ block given the data defined in the \stan+data+ block~\cite{carpenter2017stan,GorinovaGS19}.
A Stan program can thus be viewed as a function from parameters and data to the value of a special variable \stan+target+ that represents the log-density of the model.
A Stan model can be described using classic imperative statements, plus two special statements that modify the value of \stan{target}.
The first one,~\mbox{\stan|target+= |$\,e$}, increments the value of \stan+target+ by~$e$.
The second one, \mbox{\textit{e} \texttt{\textasciitilde} \textit{D}}, is equivalent to \mbox{\stan{target}$\,$\stan|+=|$\,D_{\text{lpdf}}(\textit{e})$}~\cite{GorinovaGS19} where $D_{\text{lpdf}}$ denotes the log probability density function of~$D$.

Unfortunately, these constructs allow the definition of models that cannot be translated using the generative translation defined above.
\cref{tab:stan_features} lists the Stan features that are not handled correctly.
A \emph{left expression} is where the left-hand-side of \stan{~} is an arbitrary expression.
\emph{Multiple updates} occur when the same parameter appears on the left-hand-side of multiple \stan{~} statements.
An \emph{implicit prior} occurs when there is no explicit \stan{~} statement in the model for a parameter.

The ``\%'' column of \cref{tab:stan_features} indicates the percentage of Stan models that exercise each of the non-generative features among the 531 valid files in \url{https://github.com/stan-dev/example-models}.
The example column contains illustrative excerpts from such models.
Since these are official and long-standing examples, we assume that they use the non-generative features on purpose.
Comments in the source code further corroborate that the programmer knowingly used the features.
While some features only occur in a minority of models, their prevalence is too high to ignore.

\subsection{Comprehensive translation}\label{sec:comprehensive_translation}

The previous section illustrates that Stan is centered around the definition of \stan{target}, not around generating samples for parameters, which is required by generative PPLs.
The comprehensive translation adds an initialization step to generate samples for all the parameters and compiles all Stan \stan{~} statements as observations.
Parameter initialization draws from the uniform distribution in their definition domain.
For the biased coin example, the result of this translation is shown in \Cref{fig:compile_coin_comprehensive}:
The parameter \python{z} is first sampled uniformly on its definition domain and then conditioned with an observation.

The compilation column of \cref{tab:stan_features} illustrates
the translation of non-generative features.
Left expression and multiple updates are simply compiled into observations.
Parameter initialization uses the uniform distribution over its definition domain.
For unbounded domains, we introduce new distributions (e.g., \python{improper_uniform}) with a constant density that can be normalized away during inference.
\Cref{sec:formal_comprehensive_translation} details the complete compilation scheme.

\paragraph{Intuition of correctness}
The semantics of Stan as described in \cite{GorinovaGS19} is a classic imperative semantics.  Its environment includes the special variable \stan{target}, the unnormalized log-density of the model.
On the other hand, the semantics of a generative PPL as described in \cite{staton_2017} defines a kernel mapping an environment to a measurable function.
Our compilation scheme adds uniform initializations for all parameters which comes down to the Lebesgue measure on the parameters space, and translates all \stan{~} statements to \python{observe} statements.
We can then show that a succession of  \python{observe} statements yields a distribution with the same log-density as the Stan semantics.
\Cref{sec:correctness} details the correctness proof.

\paragraph{Implementation.}
The comprehensive compilation scheme can compile any Stan program to a
generative PPL.  \Cref{sec:implem} discusses the implementation of two new backends for the Stanc3 compiler targeting Pyro~\cite{bingham_et_al_2019} -- a PPL in the line of WebPPL~\cite{goodman_stuhlmuller_2014} -- and NumPyro -- a JAX~\cite{jax2018github} based version of Pyro.
\Cref{sec:evaluation} experimentally validates that our backends can compile most existing Stan models. 
Results also show that models compiled using our NumPyro backend outperform their Stan counterpart on existing benchmarks.

\paragraph{Extensions.}
Pyro is a \emph{deep universal probabilistic programming languages}
with native support for variational inference. 
Building on Pyro, we use our compiler to extend Stan with support for explicit variational
guides (\Cref{sec:guides}) and deep neural networks to capture complex
relations between parameters (\Cref{sec:nn}).

 \section{Semantics and Compilation}
\label{sec:formal}

This section, building on previous work, first formally defines the semantics of Stan (\Cref{sec:sem_stan}) and the semantics of GProb, a small generative probabilistic language (\Cref{sec:formal_generative_language}).
It then defines the compilation function from Stan to GProb (\Cref{sec:formal_comprehensive_translation}) and proves its correctness (\Cref{sec:correctness}).

\subsection{Stan: a Declarative Probabilistic Language}
\label{sec:sem_stan}

\newcommand{\kwf}[1]{\ensuremath{\mathtt{{\color{blue}{#1}}}}}
\newcommand{\ttf}[1]{\ensuremath{{\texttt{#1}}}}

\newcommand{\pfunctions}[1]{\ensuremath{\kwf{functions}\ \ttf{\{}{#1}\ttf{\}}}}
\newcommand{\pdata}[1]{\ensuremath{\kwf{data}\ \ttf{\{}{#1}\ttf{\}}}}
\newcommand{\ptransdata}[1]{\ensuremath{\kwf{transformed\ data}\ \ttf{\{}{#1}\ttf{\}}}}
\newcommand{\pparameters}[1]{\ensuremath{\kwf{parameters}\ \ttf{\{}{#1}\ttf{\}}}}
\newcommand{\ptransparam}[1]{\ensuremath{\kwf{transformed\ parameters}\ \ttf{\{}{#1}\ttf{\}}}}
\newcommand{\pmodel}[1]{\ensuremath{\kwf{model}\ \ttf{\{}{#1}\ttf{\}}}}
\newcommand{\pgenerated}[1]{\ensuremath{\kwf{generated\ quantities}\ \ttf{\{}{#1}\ttf{\}}}}

\newcommand{\secdata}[1]{\ensuremath{\mathit{data}(#1)}}
\newcommand{\secparams}[1]{\ensuremath{\mathit{params}(#1)}}
\newcommand{\secmodel}[1]{\ensuremath{\mathit{model}(#1)}}

\newcommand{\sassign}[2]{\ensuremath{{#1}\ \ttf{=}\ {#2}}}
\newcommand{\sseq}[2]{\ensuremath{{#1}\ttf{;}\ {#2}}}
\newcommand{\sfor}[4]{\ensuremath{\kwf{for}\ \ttf{(}{#1}\ \kwf{in}\ {#2}\ttf{:}{#3}\ttf{)}\ \ttf{\{}{#4}\ttf{\}}}}
\newcommand{\sforeach}[3]{\ensuremath{\kwf{for}\ \ttf{(}{#1}\ \kwf{in}\ {#2}\ttf{)}\ \ttf{\{}{#3}\ttf{\}}}}
\newcommand{\swhile}[2]{\ensuremath{\kwf{while}\ \ttf{(}{#1}\ttf{)}\ \ttf{\{}{#2}\ttf{\}}}}
\newcommand{\sif}[3]{\ensuremath{\kwf{if}\ \ttf{(}{#1}\ttf{)}\ {#2}\ \kwf{else}\ {#3}}}
\newcommand{\sskip}{\ensuremath{\kwf{skip}}}
\newcommand{\stilde}[2]{\ensuremath{{#1}\ \mbox{\texttt{\textasciitilde}}\ {#2}}}
\newcommand{\starget}[1]{\ensuremath{\kwf{target}\ \ttf{+=}\ {#1}}}

\newcommand{\earray}[1]{\ensuremath{\ttf{\{}{#1}\ttf{\}}}}
\newcommand{\evector}[1]{\ensuremath{\ttf{[}{#1}\ttf{]}}}
\newcommand{\eaccess}[2]{\ensuremath{{#1}\ttf{[}{#2}\ttf{]}}}
\newcommand{\ecall}[2]{\ensuremath{{#1}\ttf{(}{#2}\ttf{)}}}

\newcommand{\elet}[3]{\ensuremath{\kwf{let}\,{#1}\ \ttf{=}\ {#2}\ \kwf{in}}\ {#3}}
\newcommand{\efor}[5]{\ensuremath{\kwf{for}_{#1}\ \ttf{(}{#2}\ \kwf{in}\ {#3}\ttf{:}{#4}\ttf{)}\ {#5}}}
\newcommand{\eforeach}[4]{\ensuremath{\kwf{for}_{#1}\ \ttf{(}{#2}\ \kwf{in}\ {#3}\ttf{)}\ {#4}}}
\newcommand{\ewhile}[3]{\ensuremath{\kwf{while}_{#1}\ \ttf{(}{#2}\ttf{)}\ {#3}}}
\newcommand{\eif}[3]{\ensuremath{\kwf{if}\ \ttf{(}{#1}\ttf{)}\ {#2}\ \kwf{else}\ {#3}}}
\newcommand{\efactor}[1]{\ensuremath{\kwf{factor}\ttf{(}{#1}\ttf{)}}}
\newcommand{\eobserve}[2]{\ensuremath{\kwf{observe}\ttf{(}{#1}\ttf{,}{#2}\ttf{)}}}
\newcommand{\esample}[1]{\ensuremath{\kwf{sample}\ttf{(}{#1}\ttf{)}}}
\newcommand{\ereturn}[1]{\ensuremath{\kwf{return}\ttf{(}{#1}\ttf{)}}}

\newcommand{\eunit}{\ensuremath{\ttf{()}}}

\newcommand{\aletin}[2]{\ensuremath{\mathit{let}\,{#1}\,=\,{#2}\,\mathit{in}}}
\newcommand{\aif}[3]{\ensuremath{\mathit{if}\,{#1}\,\mathit{then}\,{#2}\,\mathit{else}\,{#3}}}

\newcommand{\treal}{\ensuremath{\kwf{real}}}
\newcommand{\tint}{\ensuremath{\kwf{int}}}
\newcommand{\treals}{\ensuremath{\kwf{reals}}}
\newcommand{\tints}{\ensuremath{\kwf{ints}}}
\newcommand{\tvector}[1]{\ensuremath{\kwf{vector}\ttf{[}{#1}\ttf{]}}}
\newcommand{\trowvector}[1]{\ensuremath{\kwf{row\_vector}\ttf{[}{#1}\ttf{]}}}
\newcommand{\tmatrix}[2]{\ensuremath{\kwf{matrix}\ttf{[}{#1}\ttf{,}{#2}\ttf{]}}}
\newcommand{\tarray}[2]{\ensuremath{{#1}\ttf{[}{#2}\ttf{]}}}

\newcommand{\dom}[1]{\ensuremath{\textrm{Dom}({#1})}}

\newcommand{\ImproperUniform}{\ensuremath{\mathtt{improper\_uniform}}}

\newcommand{\stanunroll}[2]{\ensuremath{\left\lfloor #1 \right\rfloor_{#2}}}
\newcommand{\unroll}[2]{\ensuremath{\left\lfloor #1 \right\rfloor_{#2}}}

The Stan language is informally described in~\cite{carpenter2017stan}.
A Stan program is a sequence of blocks which in order:
declares functions, declares input names and types, pre-processes input data, declares the parameters to infer, defines transformations on parameters, defines the model, and post-processes the parameters.
The only mandatory block is \stan{model}.
Variables declared in a block are visible in subsequent blocks.
\begin{small}
\begin{center}
  \begin{tabular}{lcl}
 \textit{program} & ::= &
  \pfunctions{\mathit{fundecl}^*} ?
\\ &&
  \pdata{\mathit{decl}^*} ?
\\ &&
  \ptransdata{\mathit{decl}^*\ \mathit{stmt}} ?
\\ &&
  \pparameters{\mathit{decl}^*} ?
\\ &&
  \ptransparam{\mathit{decl}^*\ \mathit{stmt}} ?
\\ &&
  \pmodel{\mathit{decl}^*\ \mathit{stmt}}
\\ &&
  \pgenerated{\mathit{decl}^*\ \mathit{stmt}} ?
\end{tabular}
\end{center}
\end{small}

Variable declarations (\textit{decl}$^*$) are lists of variables names with their types (e.g., \stan{int N;}) or arrays with their type and shape (e.g., \stan{real x[N]}).
Types can be constrained to specify the domain of a variable (e.g., \stan{real <lower=0> x} for \mbox{$x \in \mathbb{R}^+$}).
Note that \stan{vector} and \stan{matrix} are primitive types that can be used in arrays (e.g. \stan{vector[N] x[10]} is an array of \stan{10} vectors of size \stan{N}).
Shapes and sizes of arrays, matrices, and vectors are explicit and can be arbitrary expressions. 
\begin{small}
\begin{center}
\def\arraystretch{1.2}
\begin{tabular}{l@{$\ $}c@{$\ $}l}
\textit{decl}& ::= &
  \textit{base\_type} \textit{constraint} $x$ \ttf{;}\\
&$~\mid~$&
  \textit{base\_type} \textit{constraint} $x$ \ttf{[}\textit{shape}\ttf{]} \ttf{;}
\\
\textit{base\_type}&::= &
  \treal $~\mid~$
                          \tint\\
&$~\mid~$&
\tvector{\textit{size}} $~\mid~$
  \tmatrix{\textit{size}}{\textit{size}}
\\
\textit{constraint}&::= &
  $\varepsilon$ $~\mid~$
  \ttf{<} \ttf{lower} \ttf{=} $e$ \ttf{,} \ttf{upper} \ttf{=} $e$ \ttf{>}\\
&$~\mid~$&
  \ttf{<} \ttf{lower} \ttf{=} $e$ \ttf{>} $~\mid~$
  \ttf{<}  \ttf{upper} \ttf{=} $e$ \ttf{>}
\\
\textit{shape} & ::= &
\textit{size} $~\mid~$ \textit{shape} \ttf{,} \textit{size}
\\
\textit{size} & ::= &
  \textit{e}
\end{tabular}
\end{center}
\end{small}

Inside a block, Stan is similar to a classic imperative language, with
two extra, specialized statements: $\starget{e}$ directly updates the log-density of the model (stored in the reserved variable \stan{target}), and $\stilde{x}{D}$ indicates that a variable~$x$ follows a distribution~$D$.

\begin{small}
\begin{center}
\begin{tabular}{l@{\ \ }r@{\ \ }l@{\ \ }l}
\textit{stmt} & ::= &
\sassign{x}{e}
& variable assignment \\ & $\mid$ &
\sassign{\eaccess{x}{e_1\ttf{,} ...\ttf{,} e_n}}{e}
& array assignment \\ & $\mid$ &
\sseq{\mathit{stmt_1}}{\mathit{stmt_2}}
& sequence \\ & $\mid$&
\sfor{x}{e_1}{e_2}{\mathit{stmt}}
&  loop over a range \\ & $\mid$&
\sforeach{x}{e}{\mathit{stmt}}
& loop over a collection \\ & $\mid$&
\swhile{e}{\textit{stmt}}
& while loop \\ & $\mid$&
\sif{e}{\mathit{stmt_1}}{\mathit{stmt_2}}
& conditional \\ & $\mid$&
\sskip
& no operation \\ & $\mid$&
\starget{e} & direct log-density update \\ & $\mid$ &
\stilde{e}{\ecall{f}{e_1\ttf{,}...\ttf{,}e_n}}
& probability distribution
\end{tabular}
\end{center}
\end{small}

Expressions comprise constants, variables, arrays, vectors, matrices, access to elements of an indexed structure, and function calls (also used to model binary operators):
\begin{small}
\begin{center}
\begin{tabular}{l@{~~}c@{~~}l@{$\quad$}l}
\textit{e} & ::= &
  $c$
$\mid$
  $x$
$\mid$
  \ecall{f}{e_1\ttf{,}...\ttf{,}e_n}
  $\mid$
  \earray{e_1\ttf{,}...\ttf{,}e_n}
  $\mid$
  \evector{e_1\ttf{,}...\ttf{,}e_n}\\
  &$\mid$&
  \evector{\evector{e_{1_1}\ttf{,}...\ttf{,}e_{1_m}}\ttf{,}...\ttf{,}\evector{e_{n_1}\ttf{,}...\ttf{,}e_{n_m}}}
$\mid$
  \eaccess{e_1}{e_2}
\end{tabular}
\end{center}
\end{small}

\paragraph{Semantics.}
Stan programs are evaluated in three steps:
\begin{enumerate}
\item data preparation with \stan{data} and \stan{transformed} \stan{data}
\item inference over the model defined by \stan{parameters},\\
  \stan{transformed} \stan{parameters}, and \stan{model}
\item post-processing with \stan{generated} \stan{quantities}.
\end{enumerate}

\noindent
Sections \stan{transformed} \stan{data}, \stan{transformed} \stan{parameters}, and \stan{generated} \stan{quantities} introduce new variables.
Semantically, these sections can all be inlined in the \stan{model} section.
Any Stan program can thus be rewritten into an equivalent program with only the three blocks \stan{data}, \stan{parameters}, and \stan{model}.
Alternatively, we show in \Cref{sec:compil-all}  that the \stan{transformed} \stan{data} section can be pre-computed and passed as input to the model, and the \stan{generated} \stan{quantities} can be post-processed after the inference.

\begin{small}
$$
\begin{array}{lcl}
\pfunctions{\mathit{fundecls}} \\
\pdata{\mathit{decls}_d}                         &&  \\
\kwf{transformed\ data}\\
\ttf{\{}\mathit{decls}_{\mathit{td}}~\mathit{stmt}_{\mathit{td}} \ttf{\}}         \\
\pparameters{\mathit{decls}_p}                   &&\pdata{\mathit{decls}_d} \\
  \kwf{transformed\ parameters} & \equiv    &\pparameters{\mathit{decls}_p}\\
  \qquad\qquad\ttf{\{}{\mathit{decls}_{\mathit{tp}}~\mathit{stmt}_{\mathit{tp}}}\ttf{\}} && \kwf{model}\ \ttf{\{}\\
\pmodel{\mathit{decls}_m~\mathit{stmt}_m}        &&      \quad \mathit{decls}_{\mathit{td}}~\mathit{decls}_{\mathit{tp}}~\mathit{decls}_m~\mathit{decls}_g \\
  \kwf{generated\ quantities}&& \quad \mathit{stmt}_{\mathit{td}}'~\mathit{stmt}_{\mathit{tp}}'~\mathit{stmt}_m'~\mathit{stmt}_g' \\
  \qquad\qquad\ttf{\{}{\mathit{decls}_g~\mathit{stmt}_g}\ttf{\}}&&\ttf{\}}\\
\end{array}
$$
\end{small}

\noindent
Functions declared in \stan{functions} are inlined~($\mathit{stmt}'$ is equivalent to $\mathit{stmt}$ after inlining).
To simplify the presentation, we focus on this simplified language.

\paragraph{Notations.}
To refer to the different parts of a program, we will use the following functions.
For a Stan program $p = \pdata{\mathit{decls}_d}~\pparameters{\mathit{decls}_p}~\pmodel{\mathit{decls}_m~\mathit{stmt}}$:
\[  \secdata{p} = \mathit{decls}_d \quad
  \secparams{p} = \mathit{decls}_p \quad
  \secmodel{p} = \mathit{stmt}
\]

In the following, an environment $\gamma: \mathit{Var} \rightarrow \mathit{Val}$ is a mapping from variables to values,  
$\gamma(x)$ returns the value of the variable~$x$ in an environment~$\gamma$, $\gamma[x \leftarrow v]$ returns the environment $\gamma$ where the value of~$x$ is set to~$v$, and $\gamma_1, \gamma_2$ denotes the union of two environments.

The notation $\int_X \mu(dx) f(x)$ is the integral of~$f$ w.r.t. the measure~$\mu$ where~$x \in X$ is the integration variable. When $\mu$ is the Lebesgue measure we also write  $\int_X f(x) dx$.

\medskip
Following~\cite{GorinovaGS19}, we define the semantics of the model block as a deterministic function that takes an initial environment containing the input data and the parameters, and returns an updated environment where the value of the variable \stan{target} is the un-normalized log-density of the model.

We can then define the semantics of a Stan program as a \emph{kernel}~\cite{kozen_1981, staton_et_al_2016, staton_2017}, that is, a function $\psem{p}: \mathcal{D} \rightarrow \Sigma_{X} \rightarrow [0, \infty]$ where $\Sigma_{X}$ denotes the $\sigma$-algebra of the parameter domain~$X$, that is, the set of measurable sets of the product space of parameter values.
Given an environment~$D$ containing the input data, $\sem{p}_D$ is a \emph{measure} that maps a measurable set of parameter values~$U$ to a score in $[0, \infty]$ obtained by integrating the density of the model, $\exp(\kwf{target})$, over all the possible parameter values in~$U$.
\begin{small}
$$
\psem{p}_D
~=~
\lambda U. \int_U \exp(\sem{\secmodel{p}}_{D[\secparams{p} \leftarrow \theta]}(\kwf{target}))\ d\theta
$$
\end{small}

\noindent
Given the input data, the posterior distribution of a Stan program is obtained by normalizing the measure $\psem{p}_D$.
As the integrals are often intractable, the runtime uses approximate inference schemes to compute the posterior distribution.

\medskip

We now detail the semantics of statements and expressions in a model block.
This formalization is similar to the semantics proposed in~\cite{GorinovaGS19} but expressed denotationally.

\begin{figure}[t]
  \vspace{-0.5em}
  \begin{small}
  $$
  \def\arraystretch{1.4}
  \hspace*{-3mm}\begin{array}{l@{\;}c@{\;\;}l}
  \sem{\sassign{x}{e}}_\gamma &=&
    \gamma[x \leftarrow \sem{e}_\gamma]
  \\
  \sem{\sassign{\eaccess{x}{e_1\ttf{,} ...\ttf{,} e_n}}{e}}_\gamma &=&
    \gamma[x \leftarrow (\eaccess{x}{\sem{e_1}_\gamma\ttf{,} ...\ttf{,} \sem{e_n}_\gamma} \leftarrow \sem{e}_\gamma) ]
  \\
  \sem{\sseq{s_1}{s_2}}_\gamma &=&
    \sem{s_2}_{\sem{s_1}_\gamma}
  \\
  \multicolumn{3}{l}{
  \sem{\sfor{x}{e_1}{e_2}{s}}_\gamma =
  }\\[0.em]
  \multicolumn{3}{l}{\qquad
    {\def\arraystretch{1.1}\begin{array}[t]{@{}l}
      \aletin{n_1}{\sem{e_1}_\gamma} \ 
      \aletin{n_2}{\sem{e_2}_\gamma} \\
      \mathit{if}\ {n_1 > n_2}\ \mathit{then}\ {\gamma}\ 
      \mathit{else}\ \sem{\sfor{x}{n_1 + 1}{n_2}{s}}_{\sem{s}_{\gamma[x \leftarrow n_1]}}\\
    \end{array}}
  }
  \\
  \sem{\swhile{e}{s}}_\gamma &=&
    {\def\arraystretch{1.1}\begin{array}[t]{@{}l}
      \mathit{if}\ \sem{e}_\gamma = 0\ \mathit{then}\ {\gamma}\
      \mathit{else}\ \sem{\swhile{e}{s}}_{\sem{s}_\gamma}
    \end{array}}
  \\
  \sem{\sif{e}{s_1}{s_2}}_\gamma &=&
    \aif{\sem{e_1}_\gamma \not= 0}{\sem{s_1}_\gamma}{\sem{s_2}_\gamma}
  \\
  \sem{\sskip}_\gamma &=&
    \gamma
  \\
  \sem{\starget{e}}_\gamma &=&
    \gamma[\kwf{target} \leftarrow \gamma(\kwf{target}) + \sem{e}_\gamma]
  \\
  \sem{\stilde{e_1}{e_2}}_\gamma &=&
\aletin{D}{\sem{e_2}_\gamma}
      \sem{\starget{D_{\mathit{lpdf}}(e_1)}}_\gamma
\end{array}  $$
  \end{small}
  \vspace{-1.em}
  \caption{Semantics of statements}
  \label{fig:ssem}
\end{figure}

\paragraph{Statements.}
The semantics of a statement $\sem{s}: (\mathit{Var} \rightarrow \mathit{Val}) \rightarrow (\mathit{Var} \rightarrow \mathit{Val})$ is a function from an environment $\gamma$ to an updated environment.
\Cref{fig:ssem} gives the semantics of Stan statements.
The initial environment contains the input data, the parameters, and the reserved variable \stan{target} initialized to \stan{0}.
An assignment updates the value of a variable or of a cell of an indexed structure in the environment.
A sequence $\sseq{s_1}{s_2}$ evaluates $s_2$ in the environment produced by $s_1$.
A \stan{for} loop on ranges first evaluates the value of the bounds~$n_1$ and~$n_2$ and then repeats the execution of the body $1 + n_2 - n_1$ times.
Iterations over indexed structures~($\sforeach{x}{e}{s}$) are syntactic sugar over loops on ranges.
The behavior depends on the underlying type.
For vectors and arrays, iteration is limited to one dimension.
\begin{small}
$$
\begin{array}{l@{\;}c@{\;\;}l}
  \sem{\sforeach{x}{e}{s}}_\gamma &=&
    \begin{array}[t]{@{}l}
      \aletin{v}{\sem{e}_\gamma} \qquad \text{($i$ is a fresh variable)}\\
      \sem{\sfor{i}{1}{\mathit{length}(v)}{\sseq{\sassign{x}{\eaccess{v}{i}}}{s}}}_\gamma
    \end{array}
\end{array}
$$
\end{small}
For matrices, iteration is over the two dimensions:
\begin{small}
$$
\begin{array}{l}
  \sem{\sforeach{x}{e}{s}}_\gamma = \qquad\qquad\text{($i$ and $j$ are fresh variables)}\\ \qquad
    \begin{array}[t]{@{}l@{}}
      \aletin{v}{\sem{e}_\gamma} \\
      \sem{\begin{array}{@{}l@{}}
        \efor{}{i}{1}{\mathit{length}(v)}{} \\
            {\qquad \sfor{j}{1}{\mathit{length}(\eaccess{v}{i})}{\sseq{\sassign{x}{\eaccess{\eaccess{v}{i}}{j}}}{s}}}
      \end{array}}_\gamma
    \end{array}
\end{array}
$$
\end{small}

\noindent
A \stan{while} loop repeats the execution of its body while the condition is not~$0$.
An \stan{if} statement executes one of the two branches depending on the value of the condition.
A \stan{skip} leaves the environment unchanged.
A statement $\starget{e}$ adds the value of~$e$ to \stan{target} in the environment.
Finally, a statement $\stilde{e_1}{e_2}$ evaluates the expression~$e_2$ into a probability distribution~$D$ and updates the target with the value of the log-density of~$D$ at~$e_1$.

\begin{figure}[t]
  \centering
  \begin{small}
  $$
  \def\arraystretch{1.4}
  \begin{array}{r@{\;}c@{\;}l@{\;\;}r@{\;}c@{\;}l}
  \sem{c}_\gamma &=& c
  &
  \sem{\earray{e_1\ttf{,}...\ttf{,}e_n}}_\gamma &=&
    \earray{\sem{e_1}_\gamma\ttf{,}...\ttf{,}\sem{e_n}_\gamma}
  \\
  \sem{x}_\gamma &=& \gamma(x)
  &
  \sem{\evector{e_1\ttf{,}...\ttf{,}e_n}}_\gamma &=&
    \evector{\sem{e_1}_\gamma\ttf{,}...\ttf{,}\sem{e_n}_\gamma}
  \\
  \sem{\eaccess{e_1}{e_2}}_\gamma &=& \eaccess{\sem{e_1}_\gamma}{\sem{e_2}_\gamma}
  &
  \sem{\ecall{f}{e}}_\gamma &=& f(\sem{e}_\gamma)
  \end{array}
  $$
\end{small}
  \vspace{-1.em}
  \caption{Semantics of expressions}
  \label{fig:esem}
\end{figure}

\paragraph{Expressions.} The semantics of an expression $\sem{e}: (\mathit{Var} \rightarrow \mathit{Val}) \rightarrow \mathit{Val}$ is a function from a environment to values. 
\Cref{fig:esem} gives the semantics of Stan expressions.
Constants evaluate to themselves.  Variables are looked up in the  environment.
Arrays, vectors, and matrix expressions evaluate all their components.
Indexing expressions obtain the corresponding value in the associated data.
Function calls apply the function to the value of the arguments.
Functions are built-ins like~\stan{+} or \stan{normal} (user-defined functions are inlined).

\paragraph{Limitations}
We consider only terminating programs which means in particular that all loops perform a bounded number of iterations.
We also limit the access and update of \stan{target} to the statements $\starget{e}$ and $\stilde{e_1}{e_2}$.

\begin{assumption}
  \label{hyp:term}
  All programs terminate.
\end{assumption}
\begin{assumption}
  \label{hyp:target}
  Expressions cannot depend on \stan{target}.
\end{assumption}

\subsection{GProb: a Simple Generative PPL}
\label{sec:formal_generative_language}

To formalize the compilation, we first define the target language: GProb, a simple generative probabilistic language  similar to the one defined in~\cite{staton_2017}.
GProb is an expression language with the following syntax:
\begin{small}
$$
\begin{array}{l@{\ \ }r@{\ \ }l}
e & ::= & 
          c \mid
          x \mid
          \earray{e_1\ttf{,}...\ttf{,}e_n} \mid
          \evector{e_1\ttf{,}...\ttf{,}e_n} \mid
          \eaccess{e_1}{e_2} \mid
          \ecall{f}{e_1\ttf{,}...\ttf{,}e_n} \\& \mid~&
          \elet{x}{e_1}{e_2} \mid
          \elet{\eaccess{x}{e_1\ttf{,} ...\ttf{,} e_n}}{e}{e'} \\& \mid~&
          \eif{e}{e_1}{e_2} \mid
          \efor{\mathcal{X}}{x}{e_1}{e_2}{e_3} \mid
          \ewhile{\mathcal{X}}{e_1}{e_2} \\& \mid~&
          \efactor{e} \mid
          \esample{e} \mid
          \ereturn{e}
\end{array}
$$
\end{small}

\noindent
An expression is either a Stan expression, a local binding~(\stan{let}), a conditional~(\python{if}), or a loop~(\python{for} or \python{while}).
To simplify the presentation, loops are parameterized by the set~$\mathcal{X}$ of variables updated and returned by their body.
Moreover, we limit the definition of the semantics to terminating loops that do not depend on sampled values in the body of the loop.
GProb also contains the classic probabilistic expressions: \stan{sample} draws a sample from a distribution, and \stan{factor} assigns a score to the current execution trace to condition the model.
The \python{return} expression lifts a deterministic expression to a probabilistic context.

We also introduce \stan{observe($D$,$v$)} as a syntactic shortcut for \stan{factor($D_{\rm{pdf}}$(v))} where $D_{\rm{pdf}}$ denotes the density function of~$D$.
This construct penalizes the current execution with the likelihood of~$v$ w.r.t.~$D$ which captures the assumption that the observed data~$v$ follows the distribution~$D$.

\begin{figure}[t]
  \begin{small}
  \centering
  \[
  \def\arraystretch{1.8}
  \begin{array}{@{}l@{\;}c@{\;\;}l}
    \psem{\ereturn{e}}_\gamma & = & \lambda U. \, \delta_{\sem{e}_\gamma}(U)
    \\

    \psem{\elet{x}{e_1}{e_2}}_\gamma & = &
      \lambda U. \displaystyle\int_X\psem{e_1}_\gamma(dv) \times \psem{e_2}_{\gamma[x \leftarrow v]}(U)
    \\

    \multicolumn{3}{@{}l}{
    \psem{\elet{\eaccess{x}{e_1\ttf{,} ...\ttf{,} e_n}}{e}{e'}}_\gamma =
    } \\[-.5em]
    \multicolumn{3}{@{}l}{\quad
      \lambda U. \displaystyle\int_{X}\psem{e}_\gamma(dv) \times \psem{e'}_{\gamma[x \leftarrow (\eaccess{x}{\sem{e_1}_\gamma\ttf{,} ...\ttf{,} \sem{e_n}_\gamma} \leftarrow v )]}(U)
    }
    \\
\multicolumn{3}{@{}l}{
    \psem{\efor{\mathcal{X}}{x}{e_1}{e_2}{e_3}}_\gamma =
    } \\[-.5em]
    \multicolumn{3}{@{}l}{\quad
      \lambda U. {\def\arraystretch{1.4}\begin{array}[t]{@{}l}
                   \aletin{n_1}{\sem{e_1}_\gamma} ~
                   \aletin{n_2}{\sem{e_2}_\gamma} \\
                   \mathit{if}\ {n_1 > n_2}\ \mathit{then}\ \delta_{\gamma(\mathcal{X})}(U)\\
                   \mathit{else}\! \displaystyle\int_X 
                   {\def\arraystretch{1.4}\begin{array}[t]{@{}l}
                   \psem{e_3}_{\gamma[x \leftarrow n_1]}(dv) \ \times \\ 
                   \psem{\efor{\mathcal{X}}{x}{n_1 + 1}{n_2}{e_3}}_{\gamma[\mathcal{X} \leftarrow v]}(U)
                  \end{array}}
                 \end{array}}
    }
    \\

    \psem{\ewhile{\mathcal{X}}{e_1}{e_2}}_\gamma & = &
\\[-.5em]
    \multicolumn{3}{@{}l}{\quad
      \lambda U. {\def\arraystretch{1.4}\begin{array}[t]{@{}l}
                   \mathit{if}\ \sem{e_1}_\gamma = 0\ \mathit{then}\ \delta_{\gamma(\mathcal{X})}(U)\\
                   \mathit{else}\! \displaystyle\int_X \psem{e_2}_\gamma(dv) \times \psem{\ewhile{\mathcal{X}}{e_1}{e_2}}_{\gamma[\mathcal{X} \leftarrow v]}(U)
                 \end{array}}
    }
    \\

    \psem{\eif{e}{e_1}{e_2}}_\gamma & = &
      \lambda U. 
      {\def\arraystretch{1.4}\begin{array}[t]{@{}l}
        \mathit{if}\ \sem{e}_\gamma \not= 0 
        \ \mathit{then}\ \psem{e_1}_\gamma(U)\\
        \mathit{else}\ \psem{e_2}_\gamma(U)
    \end{array}}
    \\

    \psem{\esample{e}}_\gamma & = & \lambda U. \sem{e}_\gamma(U) \\

    \psem{\efactor{e}}_\gamma & = & \lambda U. \exp(\sem{e}_\gamma) \delta_{\eunit}(U) \\
  \end{array}
  \]
\end{small}
  \vspace{-1.em}
  \caption{Generative probabilistic language semantics}
  \label{fig:pplsem}
\end{figure}

\paragraph{Semantics.}
Following~\cite{staton_2017} we give a measure-based semantics to GProb.
The semantics of an expression is a kernel that given an environment returns a measure on the set of possible values.
Given input data, normalizing the corresponding measure computes the program's posterior distribution.

The semantics of GProb is given in \Cref{fig:pplsem}.
A deterministic expression is lifted to a probabilistic expression with the Dirac delta measure ($\delta_x(U)= 1$ if $x \in U$, $0$ otherwise).
A local definition $\elet{x}{e_1}{e_2}$ is interpreted by integrating the semantics of $e_2$ over the set of all possible values for~$x$.

Compared to the language defined in~\cite{staton_2017}, we added Stan-like loops.
Loops behave like a sequence of expressions and return the values of the variables updated by their body.
We impose that the condition of a loop cannot depend on parameters sampled in the loop body, and consider only terminating loops.
Hence for any given context $\gamma$, it suffices to unroll the  definition of the loop semantics a finite number of times to get a measure term describing the semantics.

Finally, the semantics of probabilistic operators is the following.
The semantics of \stan{sample($e$)} is the probability distribution $\sem{e}_\gamma$ (e.g. $\mathcal{N}(0, 1)$). A type system, omitted here for conciseness, ensures that we only sample from distributions.
The semantics of \stan{factor($e$)} is a measure defined on the singleton space \eunit\ whose value is ~$\exp(\sem{e})$ (this operator corresponds to $\kwf{score}$ in~\cite{staton_2017} but in log-scale, which is common for numerical precision).

\subsection{Comprehensive Translation}\label{sec:formal_comprehensive_translation}

The key idea is to first
sample all parameters from priors with a constant density that can be normalized away during inference (e.g., $\mathit{uniform}$ on bounded domains),
and then compile all~\stan{~} statements into \stan{observe} statements.

The compilation functions~$\dcomp{k}{\secparams{p}}$ for the parameters and~$\scomp{k}{\secmodel{p}}$ for the model are both parameterized by a continuation~$k$.
The compilation of the entire program first compiles the parameters to introduce the priors, then compiles the model, and finally adds a return statement for all the parameters. In continuation passing style:
\begin{small}
$$
\pcomp{p} = \dcomp{\scomp{\ereturn{\secparams{p}}}{\secmodel{p}}}{\secparams{p}}
$$
\end{small}

\begin{figure}
  \begin{small}
  \centering
  $$
  \def\arraystretch{1.3}
  \begin{array}{@{}l@{\;}c@{\;\;}l}
    \dcomp{k}{t\ \mathit{cstr}\ x \ttf{;}}
    & = & \elet{x}{\esample{\kcomp{\mathit{cstr}}{\ttf{[]}}}}{k}
    \\
    \dcomp{k}{t\ \mathit{cstr}\ \tarray{x}{\mathit{shape}} \ttf{;}}
    & = &
    \elet{x}{\esample{\kcomp{\mathit{cstr}}{\mathit{shape}}}}{k}
    \\
    \dcomp{k}{\mathit{decl}\ \mathit{decls}}
    & = &
    \dcomp{\dcomp{k}{\mathit{decls}}}{\mathit{decl}}
    \\
    \kcomp{\varepsilon}{\mathit{shape}}
    & = & \ecall{\ImproperUniform}{[-\infty, \infty]\ttf{,}{\mathit{shape}}}
    \\
    \kcomp{\ttf{<}\ttf{lower}\ttf{=}e_1\ttf{>}}{\mathit{shape}}
    & = & \ecall{\ImproperUniform}{[e_1, \infty]\ttf{,}{\mathit{shape}}}
    \\
    \kcomp{\ttf{<}\ttf{upper}\ttf{=}e_2\ttf{>}}{\mathit{shape}}
    & = & \ecall{\ImproperUniform}{[-\infty, e_2]\ttf{,}{\mathit{shape}}}
    \\
\multicolumn{3}{@{}l}{
    \kcomp{\ttf{<}\ttf{lower}\ttf{=}e_1\ttf{,}\ttf{upper}\ttf{=}e_2\ttf{>}}{\mathit{shape}}
    =  \ecall{\mathtt{uniform}}{[e_1, e_2]\ttf{,}{\mathit{shape}}}
}
  \end{array}
  $$
  \vspace{-0.5em}
  \caption{Comprehensive compilation of parameters}
  \label{fig:comp-params}
\end{small}
\end{figure}

\begin{figure}\
  \begin{small}
  \centering
  $$
  \def\arraystretch{1.3}
  \begin{array}{@{}l@{\;}c@{\;\;}l}
    \scomp{k}{\sassign{x}{e}} & = & \elet{x}{\ereturn{e}}{k} \\
    \scomp{k}{\sassign{\eaccess{x}{e_1\ttf{,} ...\ttf{,} e_n}}{e}} & = &
      \elet{\eaccess{x}{e_1\ttf{,} ...\ttf{,} e_n}}{e}{k} \\
    \scomp{k}{\sseq{\mathit{s_1}}{\mathit{s_2}}} & = &
      \scomp{\scomp{k}{\mathit{s_2}}}{\mathit{s_1}} \\
    \multicolumn{3}{@{}l}{
    \scomp{k}{\sfor{x}{e_1}{e_2}{\mathit{s}}} =
    }\\
    \multicolumn{3}{@{}l}{\qquad
      \elet{\mathit{lhs}(\mathit{s})}{\efor{\mathit{lhs}(\mathit{s})}{x}{e_1}{e_2}{\scomp{\ereturn{\mathit{lhs}(\mathit{s})}}{\mathit{s}}}}{k}
    } \\
    \multicolumn{3}{@{}l}{
    \scomp{k}{\swhile{e}{\mathit{s}}} =
    }\\
    \multicolumn{3}{@{}l}{\qquad
      \elet{\mathit{lhs}(\mathit{s})}{\ewhile{\mathcal{\mathit{lhs}(\mathit{s})}}{e}{\scomp{\ereturn{\mathit{lhs}(\mathit{s})}}{\mathit{s}}}}{k}
    }\\
    \scomp{k}{\sif{e}{\mathit{s}_1}{\mathit{s}_2}} & = &
      \eif{e}{\scomp{k}{\mathit{s}_1}}{\scomp{k}{\mathit{s}_2}} \\
    \scomp{k}{\sskip} & = & k \\
    \scomp{k}{\starget{e}} & = &
       \elet{\eunit}{\efactor{e}}{k} \\
    \scomp{k}{\stilde{e}{\ecall{f}{e_1\ttf{,}...\ttf{,}e_n}}} & = &
        \elet{\eunit}{\eobserve{\ecall{f}{e_1\ttf{,}...\ttf{,}e_n}}{e}}{k} \\
  \end{array}
  $$
  \end{small}
  \vspace{-1.em}
  \caption{Comprehensive compilation of statements}
  \label{fig:comp}
\end{figure}

\paragraph{Parameters.}
In Stan, parameters are defined on $\mathbb{R}^n$ with optional domain constraints (e.g. \stan{<lower=0>}).
For each parameter, the comprehensive translation sets the prior to either the \emph{uniform} distribution on a bounded domain, or an improper prior with a constant density w.r.t.\ the Lebesgue measure that we call \ImproperUniform.
The compilation function of the parameters, defined \Cref{fig:comp-params}, thus produces a succession of sample expressions: 
\begin{small}
$$
\dcomp{k}{\secparams{p}} ~=~ \elet{x_1}{D_1}{\dots\elet{x_n}{D_n}{k}}
$$
\end{small}
where each~$D_i$ is either \python{uniform} or \python{improper_uniform}.

\paragraph{Statements.}
\Cref{fig:comp} defines compilation for
statements, $\scomp{k}{\mathit{stmt}}$, parameterized by a continuation~$k$.
Stan imperative assignments become functional updates using local bindings.
Compiling the sequences chains the continuations.
Updates to the target are compiled into \python{factor} expressions and all
\stan{~} statements are compiled into observations. 

The compilation of Stan loops raises an issue.
In Stan, all the variables that appear on the left-hand side of an assignment in the body of a loop are state variables that are incrementally updated at each iteration.
Since GProb is a functional language, the state of the loops is made explicit.
To propagate the environment updates at each iteration, loop expressions are annotated with all the variables that are assigned in their body~($\mathit{lhs}(\mathit{stmt})$).
These variables are returned at the end of the loop and can be used in the continuation.

\paragraph{Pre- and post-processing blocks.}
\label{sec:compil-all}
\Cref{sec:sem_stan} showed that all Stan programs can be rewritten in a kernel using only the \stan{data}, \stan{parameters}, and \stan{model} sections. This approach can make the model more complicated and thus, the inference more expensive. 
In particular, pre- and post-processing steps do not need to be computed at each inference step.

In Stan, users can define functions in the \stan{functions} block.
These functions can be compiled into functions in the target language using the comprehensive translation. 

The code in the \stan{transformed} \stan{data} section only depends on variables declared in the \stan{data} section and can be computed only once before the inference.
We compile this section into a function that takes as argument the data and returns the transformed data.
The variables declared in the \stan{transformed} \stan{data} section become new inputs for the model.

On the other hand, the \stan{transformed} \stan{parameters} block must be part of the model since it depends on the model parameters.
This section is thus inlined in the compiled model.

Finally, the \stan{generated} \stan{quantities} block is executed only once on the inference result.
It is compiled into a function of the data, transformed data, and parameters returned by the model. 
The \stan{transformed} \stan{parameters} block is also inlined, since generated quantities may depend on them.

\subsection{Correctness of the Compilation}
\label{sec:correctness}

We can now show that a Stan program and the corresponding compiled code yield the same un-normalized measure up to a constant factor (and thus the same posterior distribution).
The proof has two steps: (1)~simplifying the sequence of \python{sample} statements introduced by the compilation of the parameters, and (2)~showing that the value of the Stan \stan{target} corresponds to the score computed by the generated code.

\paragraph{Priors.}
First, we simplify the nested integrals introduced by the sequence of \python{sample} statements for the parameters priors into one integral over the parameter domain.

\begin{lemma}
  \label{lem:params}
  For all Stan programs $p$ with $\mathit{stmt} = \secmodel{p}$ and $\mathcal{P} = \secparams{p}$, and environments~$\gamma$:
\begin{small}
  $$
  \psem{\pcomp{p}}_\gamma \propto
    \lambda U. \!\!\int_U\! \psem{\scomp{\ereturn{\eunit}}{\mathit{stmt}}}_{\gamma[\mathcal{P} \leftarrow \theta]}(\{\eunit\})d\theta
  $$
\end{small}

\noindent
where $U \in \Sigma_{X}$ is a measurable set of parameter values, with $X = \dom{\mathcal{P}}$.
\end{lemma}

The proof relies on the fact that the parameters are sampled from the distributions \python{uniform} or \python{improper_uniform} which both have constant density w.r.t. the Lebesgue measure on their domain.
These constants introduce a constant ratio~($\propto$) between the two measures.
In addition, since parameters cannot appear on the left-hand side of an assignment we can simplify the \python{return} statement.
The detailed proof is given in \ifextended \Cref{apx:correctness}\else \cite{deepstan-extended}\fi.

\paragraph{Score and \texttt{target}.}
We now show that the value of the Stan \stan{target} variable corresponds to the score computed by the generated code.

\begin{lemma}
  \label{lem:body}
  For all Stan statements $\mathit{stmt}$ compiled with a continuation $k$,
  if $\gamma(\kwf{target}) = 0$, and $\sem{\mathit{stmt}}_\gamma = \gamma'$,
$$
  \psem{\scomp{k}{\mathit{stmt}}}_{\gamma} = \lambda U.
  \exp(\gamma'(\kwf{target})) \times \psem{k}_{\gamma'[\kwf{target} \leftarrow 0]}(U)
  $$
\end{lemma}

The proof is done by induction on the structure of~$\mathit{stmt}$ and the finite number of loops iterations~(\Cref{hyp:term}).
The hypothesis $\gamma(\kwf{target}) = 0$ simplifies the induction by avoiding to keep an accumulator of the value of $\kwf{target}$.
Resetting the value of target in the environment $\gamma'[\kwf{target} \leftarrow 0]$ for the evaluation of the continuation~$k$ is thus necessary for the inductive step.
The proof is given in \ifextended \Cref{apx:correctness}\else \cite{deepstan-extended}\fi.

\paragraph{Correctness.}
We now have all the elements to prove that the comprehensive compilation is correct.
That is, generated code yields the same un-normalized measure up to a constant factor that will be normalized away by the inference.

\begin{theorem}
For all Stan programs $p$, the semantics of the source and compiled programs are equal up to a constant:

\begin{small}
$$
\psem{p}_D \propto \psem{\pcomp{p}}_{D}
$$
\end{small}
\end{theorem}
\begin{proof}
  The proof is a direct consequence of \Cref{lem:params,lem:body} and the definition of the two semantics.
  With $\mathit{stmt} = \secmodel{p}$ and $\mathcal{P} = \secparams{p}$:

  \begin{small}
  $$
  \def\arraystretch{1.8}
  \begin{array}{@{}l}
    \psem{\pcomp{p}}_{D} \propto
    \lambda U.\!\!
    \displaystyle\int_U\!\psem{\scomp{\ereturn{\eunit}}{\mathit{stmt}}}_{D[\mathcal{P} \leftarrow \theta]}(\{\eunit\})\ d\theta
    \\ =
    \lambda U.\!\!
    \displaystyle\int_U\!
                    {\def\arraystretch{1.4}\begin{array}[t]{@{}l}
                       \exp(\sem{\mathit{stmt}}_{D[\mathcal{P} \leftarrow \theta]}(\kwf{target})) 
                       \times \psem{\ereturn{\eunit}}(\{\eunit\})\ d\theta
                     \end{array}}
    \\ =
    \lambda U.\!\!
    \displaystyle\int_U\!\exp(\sem{\mathit{stmt}}_{D[\mathcal{P} \leftarrow \theta]}(\kwf{target}))\ d\theta
    \;\; = \;
    \psem{p}_{D}
    \vspace{-0.75em}
  \end{array}
  $$
\end{small}
\end{proof}

 \section{Implementation}
\label{sec:implem}

We implemented two new backends for the Stan compiler targeting Pyro~\cite{bingham_et_al_2019} and NumPyro~\cite{numpyro_2019}.
NumPyro is a variant of Pyro built on top of JAX~\cite{jax2018github}, a Python library that provides efficient automatic differentiation, vectorization, and just-in-time compilation on CPU, GPU, and TPU.

For both backends, we have implemented three compilation schemes:
generative~(\Cref{sec:generative_translation}),
comprehensive~(\Cref{sec:comprehensive_translation}),
and mixed.

\paragraph{Mixed Compilation}
The \emph{mixed} compilation scheme is an optimization of the comprehensive translation where proper priors are used whenever possible.
The mixed compilation can thus generate code that is similar to the generative translation whenever possible.

The mixed translation can be decomposed into three steps:
first, compile the program with the comprehensive scheme;
second, using the commutativity theorem of~\cite{staton_2017}, reschedule $\esample{\mathit{uniform}}$ statements as late as possible and reschedule $\eobserve{D}{x}$ statements as early as possible; and
third, merge consecutive \python{sample} and \python{observe} statements using the following property:
\begin{small}
$$
\begin{array}{@{}ll}
\elet{x}{\esample{\mathit{uniform}}}{\elet{\eunit}{\eobserve{D}{x}}{e}}
\\  \qquad \equiv
\elet{x}{\esample{\mathit{D}}}{e}
\end{array}
$$
\end{small}

In Stan, distribution are automatically truncated based on the parameter support as in the following model.
\begin{lststan}
parameters { real<lower=0> sigma; }
model { sigma ~ normal(0, 1); }
\end{lststan}

\noindent
The merge between the \python{sample} and \python{observe} statements is thus only correct if the two distributions have the same support.
We extended the signature of Stan distributions to include the definition domain which can be used to check if the merge is possible.

Finally, the mixed compilation generates correct code even if the~\stan{~} statements do not respect the dependency order such as in the following example.
\begin{lststan}
  y ~ normal(x, 1); x ~ normal(0, 1); ...
\end{lststan}
The mixed compilation reschedules the statements to generate the following code which does not break any environment update:
\begin{lststan}
  let x = sample(normal(0, 1)) in
  let y = sample(normal(x, 1)) in ...
\end{lststan}

\paragraph{Architecture}
We implemented the compiler as a fork of the Stanc3 compiler\footnote{\url{https://github.com/stan-dev/stanc3}}, thus guarantying compatibility with the official Stan syntax and existing static analyses.
The Stanc3 compiler is composed of multiple intermediate languages.
We decided to implement the new backends for the first internal language which is the closest to the Stan source.
The implementation is thus closer to the formalization, making it easier to keep track of the correspondence between the Stan source and the generated Pyro and NumPyro code.
In particular, in Pyro and NumPyro all sampling sites~(corresponding to \stan{sample} and \stan{observe}) are associated to a unique name which can be used for diagnostics and results analyses.
We use Stan variable names with an optional postfix to preserve uniqueness when necessary.
For example, in loops, the postfix tracks the current iteration.

\paragraph{Compiling to Pyro.}
The compiler addresses common challenges like name handling.
Pyro and Stan naming conventions are different (e.g., \stan{lambda} is a common parameter name in Stan) and Pyro has a shared namespace, while Stan distinguishes variables and functions.
The compiler carefully avoids conflicts by renaming.
Moreover, Stan supports function overloading.
The compiler uses static type information to disambiguate function calls by renaming.
Finally, there are semantics differences like one-based vs. zero-based arrays.

Stan has a large standard library that also has to be ported to Pyro.
Our implementation currently supports a substantial portion of, but not the entire, standard library.
Even though Pyro is built on top of Python and thus benefits from a large set of packages, it is not straightforward to implement all Stan functions.
The Python counterpart sometimes also has typing or semantic differences.
For example, the Stan Bernoulli distribution returns an integer and the Pyro one a float.
The differences are handled either in the library or in the compiler.
The categorical distribution, which is defined on ${[1, N]}$ in Stan and on ${[0, N-1]}$ in Pyro, illustrates both aspects.
The translation of categories is done in the library for a call to the log probability mass function:
\begin{lstpython}
def categorical_lpmf(y, theta):
  return Categorical(theta).log_prob(y - 1)
\end{lstpython}
and the compiler translates~(\stan{y ~ categorical(theta)}) into
\begin{lstpython}
  observe(categorical(theta), y - 1)
\end{lstpython}

Pyro does not require type declarations, but preserving the shape information of the indexes structures~(arrays, vectors, row vectors, matrices) is important. 
Parameter shapes are passed as arguments to the priors' initialization.

Finally, compared to GProb, Pyro is a Python library.
The compiler can thus use Python imperative features and side-effects for in-place mutation of arrays.
However, inference cannot run on models with arbitrary side-effects.
For instance, we need to introduce explicit copies when array cells are updated inside loops.

\paragraph{Compiling to NumPyro.}
The NumPyro backend shares the Pyro backend's challenges and has additional constraints coming from JAX.
Dynamic features like dynamic array slices are not supported and will fail during inference.
Control structures~(conditional and loops) are library functions where the body must be passed in as a pure function.
Our compiler accomplishes this by lambda-lifting the bodies of the control structures.
This is similar to the compilation of Stan loops to GProb described in \Cref{sec:formal_comprehensive_translation} where the updated variables are given explicitly.
Returning to the coin example~(\Cref{fig:coin_stan}), the compiled NumPyro code using the mixed compilation scheme is thus
\begin{lstpython}
def model(N, x):
  z = sample(beta(1, 1))
  def fori__2(i, acc):
    observe(bernoulli(z), x[i - 1])
  _ = fori_loop(1, N + 1, fori__2, None)
\end{lstpython}

NumPyro loops are a recent feature which can have a noticeable performance impact.
If the body of a loop does not contain probabilistic constructs, we thus generate a JAX loop instead of a NumPyro one.

 \section{Extending Stan: explicit variational guides and neural networks}
\label{sec:deepstan}

Probabilistic languages like Pyro offer new features to program and reason about complex models.
This section shows that our compilation scheme can be used to lift these benefits for Stan users.
Building on Pyro, we propose DeepStan, a conservative extension of Stan with:
(1) variational inference with high-level but explicit guides, and
(2) a clean interface to neural networks written in PyTorch.

\subsection{Explicit variational guides}
\label{sec:guides}

Variational Inference~(VI) tries to find the member~$q_{\theta^*}(z)$ of a family~$\mathcal{Q} = \big\{q_{\theta}(z)\big\}_{\theta \in \Theta}$ of simpler distributions that is the closest to the true posterior $p(z \por \mathbf{x})$~\cite{blei_kucukelbir_mcauliffe_2017}.
Members of the family $\mathcal{Q}$ are characterized by the values of the \emph{variational parameters}~$\theta$.
The fitness of a candidate is measured using the Kullback-Leibler~(KL) divergence from the true posterior, which VI aims to minimize:
\begin{small}
$$
q_{\theta^*} (z)= \argmin_{\theta \in \Theta} \mbox{KL}\Big(q_{\theta}(z) \matpar p(z \por \mathbf{x})\Big).
$$
\end{small}

Pyro natively supports variational inference and lets users define the
family $\mathcal{Q}$ (the \emph{variational guide}) alongside the model.
To support this for Stan users, we extend Stan with two new optional blocks: \stan{guide} \stan{parameters} and \stan{guide}.
The \stan{guide} block defines a distribution parameterized by the \stan{guide} \stan{parameters}.
Variational inference optimizes the values of these parameters to approximate the true posterior.

DeepStan inherits restrictions for the definition of the guide from Pyro:
the guide must be defined on the same parameter space as the model, i.e., it must sample all the parameters of the model;
and
the guide should also describe a distribution from which we can directly generate valid samples without running the inference first, which prevents the use of non-generative features and updates of \stan{target}.
The generative translation from \Cref{sec:generative_translation} generates a Python function that can serve as a Pyro guide.
The \stan{guide} \stan{parameters} block is used to generate Pyro \python{param} statements, which introduce learnable parameters.
Unlike Stan parameters that define random variables for use in the model, guide parameters are learnable coefficients that will be optimized during inference.

These restrictions still allow sophisticated guides.
The following section presents a guide defined by a neural network.

\subsection{Adding neural networks}
\label{sec:nn}

One important advantage of Pyro is its tight integration with PyTorch, enabling the authoring of \emph{deep probabilistic models}: probabilistic models involving neural networks.
It is impractical to define neural networks directly in Stan.
To support deep probabilistic models, we extend Stan with an optional \stan{networks} block to import neural network definitions.

Neural networks can be used to capture intricate dynamics between random variables.
An example is the \emph{Variational Auto-Encoder} (VAE) illustrated in \cref{fig:vae}.
A VAE learns a vector-space representation~$z$ for each observed data
point~$x$ (e.g., the pixels of an image)~\cite{kingma_welling_2013,rezende_mohamed_wierstra_2014}.
Each data point~$x$ depends on the latent representation~$z$ in a complex non-linear way, via a deep neural network: the \emph{decoder}.
The leftmost part of \cref{fig:vae} shows the corresponding graphical model.
The output of the decoder is a vector~$\mu$ that parameterizes a Bernoulli distribution over each dimension of~$x$ (e.g., each pixel associated to its probability of being in the image).

\begin{figure}
\begin{minipage}{0.41\linewidth}
    \hspace{-1em}
    \scalebox{0.65}{
    \begin{tikzpicture}

  \node[latent] (z1) {$z$};
  \factor[below=of z1, yshift=-0.25cm, xshift=-1cm] {decoder} {right:\texttt{decoder}} {} {};
  \node[const, left=of decoder, xshift=0.5cm] (theta) {$\theta\,$};
  \node[const, below=of decoder, yshift=0.4cm, xshift=0.25cm] (mu) {$\,\mu$};
  \factor[below=of decoder, yshift=-1cm] {bernoulli} {right:Bernoulli} {} {};
  \node[obs, below=of bernoulli, xshift=1cm, yshift=0.25cm] (x1) {$x$};

  \edge[-] {theta} {decoder};
  \edge[-latex] {z1} {decoder} ;
  \edge[-latex] {decoder} {bernoulli};
  \factoredge {} {bernoulli} {x1};

  \plate{model} { (x1)(z1)(decoder)(bernoulli)
    (decoder-caption) (bernoulli-caption)
  } {$N$};

  \node[below=of model, yshift=0.9cm, xshift=-0.25cm] (model-caption) {model $p_\theta (\mathbf{x} \por \mathbf{z})$};

  \node[latent, right=of z1] (z2) {$z$};
  \factor[below=of z2, yshift=-0.25cm, xshift=1cm] {normal} {left:Normal} {} {};
  \factor[below=of normal, yshift=-1cm] {encoder} {left:\texttt{encoder}} {} {};
  \node[const, below=of normal, yshift=0.4cm, xshift=-0.55cm] (musigma) {$\mu_z, \sigma_z$};
  \node[const, right=of encoder, xshift=-0.5cm] (phi) {$\phi$};
  \node[obs, below=of encoder, xshift=-1cm, yshift=0.25cm] (x2) {$x$};

  \edge[-] {phi} {encoder};
  \edge[-latex] {x2} {encoder};
  \edge[-latex] {encoder} {normal};
  \factoredge {} {normal} {z2};

  \plate{guide} { (x2)(z2)(encoder)(normal)
    (encoder-caption) (normal-caption)
  } {$N$};

  \node[below=of guide, yshift=0.9cm, xshift=0.25cm] (model-caption) {guide $q_\phi (\mathbf{z} \por \mathbf{x})$};
  \end{tikzpicture}}
\end{minipage}
\begin{minipage}{0.5\linewidth}
\begin{lststantablesmall}
networks {
  real[,] decoder(real[] x);
  real[,] encoder(int[,] x); }
data {
  int nz;
  int<lower=0, upper=1> x[28, 28]; }
parameters {
  real z[nz]; }
model {
  real mu[28, 28];
  z ~ normal(0, 1);
  mu = decoder(z);
  x ~ bernoulli(mu); }
guide {
  real encoded[2, nz] = encoder(x);
  real mu_z[nz] = encoded[1];
  real sigma_z[nz] = encoded[2];
  z ~ normal(mu_z, sigma_z); }
\end{lststantablesmall}
  \end{minipage}
  \vspace{-0.5em}
  \caption{Graphical models and DeepStan code of the Variational Auto-Encoder model and guide.}
  \label{fig:vae}
  \vspace{-0.5em}
\end{figure}

The key idea of the VAE is to use variational inference to learn the latent representation.
The guide maps each~$x$ to a latent variable~$z$ via another neural network: the \emph{encoder}.
The middle part of \cref{fig:vae} shows the graphical model of the guide.
The encoder returns, for each input~$x$, the parameters~$\mu_z$ and~$\sigma_z$ of a Gaussian distribution in the latent space.
Inference tries to learn good values for the parameters~$\theta$ and~$\phi$, simultaneously training the decoder and the encoder. 

The right part of \cref{fig:vae} shows the corresponding code in DeepStan.
A network is introduced similarly to an external function with its signature and must be implemented in PyTorch.
The network can be used in subsequent blocks, in particular the \stan{model} block and the \stan{guide} block.

\subsection{Bayesian networks}
\label{sec:mlp}

\begin{figure}
\begin{minipage}{0.25\linewidth}
    \hspace{-1em}
    \scalebox{0.65}{
      \begin{tikzpicture}

    \node[latent] (x) {$x$};
    \factor[below=of x, yshift=-0.5cm] {mlp} {right:\texttt{mlp}} {} {};
    \node[latent, left=of mlp] (theta) {$\theta$};
    \node[const, below=of mlp, yshift=0.4cm, xshift=0.25cm] (mu) {$\lambda$};
    \factor[below=of mlp, yshift=-1cm] {categorical} {right:Cat.} {} {};
    \node[obs, below=of categorical] (l) {$l$};

    \edge[-latex] {theta} {mlp};
    \edge[-latex] {x} {mlp} ;
    \edge[-latex] {mlp} {categorical};
    \factoredge {} {categorical} {l};

    \plate{model} { (x)(l)(mlp)(categorical)
      (mlp-caption) (categorical-caption)
    } {$N$};

    \node[below=of theta, yshift=0.9cm] (model-caption) {$p (\theta \por \mathbf{x}, \mathbf{l})$};
    \end{tikzpicture}}
  \end{minipage}
  \begin{minipage}{0.74\linewidth}
\begin{lststantablesmall}
networks { vector mlp(real[,,] imgs); }
data {
 int batch_size; int nx; int nh; int ny;
 real <lower=0, upper=1> imgs[28,28,batch_size];
 int <lower=1, upper=10> labels[batch_size]; }
parameters {
  real mlp.l1.weight[nh, nx]; real mlp.l1.bias[nh];
  real mlp.l2.weight[ny, nh]; real mlp.l2.bias[ny]; }
model {
  vector[batch_size] lambda;
  mlp.l1.weight ~  normal(0, 1);
  mlp.l1.bias ~ normal(0, 1);
  mlp.l2.weight ~ normal(0, 1);
  mlp.l2.bias ~  normal(0, 1);
  lambda = mlp(imgs);
  labels ~ categorical_logit(lambda); }
guide parameters {
  real w1_mu[nh, nx]; real w1_sigma[nh, nx];
  real b1_mu[nh]; real b1_sigma[nh];
  real w2_mu[ny, nh]; real w2_sigma[ny, nh];
  real b2_mu[ny]; real b2_sigma[ny]; }
guide {
  mlp.l1.weight ~ normal(w1_mu, exp(w1_sigma));
  mlp.l1.bias   ~ normal(b1_mu, exp(b1_sigma));
  mlp.l2.weight ~ normal(w2_mu, exp(w2_sigma));
  mlp.l2.bias   ~ normal(b2_mu, exp(b2_sigma)); }
\end{lststantablesmall}
  \end{minipage}
  \vspace{-0.5em}
  \caption{Graphical models and DeepStan code of the Bayesian MLP.}
  \label{fig:mlp}
  \vspace{-0.5em}
\end{figure}

Neural networks can also be treated as probabilistic models.
A \emph{Bayesian neural network} is a neural network whose learnable parameters (weights and biases) are random variables instead of concrete values~\cite{neal_2012}.
Building on Pyro features, we make it easy for users to \emph{lift} neural networks, i.e., replace concrete neural network parameters by random variables.

The left side of \cref{fig:mlp} shows a simple classifier for handwritten digits based on a multi-layer perceptron~(MLP) where all the parameters are lifted to random variables. Unlike the networks used in the VAE, the parameters (regrouped under the variable~$\theta$) are represented using a circle to indicate random variables.  The inference starts from prior beliefs about the parameters and learns distributions that fit observed data. We then generate samples of concrete weights and biases to obtain an ensemble of as many MLPs as desired. The ensemble can vote for predictions and can quantify agreement.

The right of \cref{fig:mlp} shows the corresponding code in DeepStan.
We let users declare lifted neural network parameters in Stan's
\stan{parameters} block just like any other random variables. Network
parameters are identified by the name of the network and a path, e.g.,
\stan{mlp.l1.weight}, following PyTorch naming conventions.
The \stan{model} block defines
\stan{normal(0,1)} priors for the weights and biases of the two linear
layers of the MLP. Then, for each image, the computed label follows a
categorical distribution parameterized by the output of the network,
which associates a probability to each of the ten possible values of
the discrete random variable~\stan{label}. The \stan{guide} \stan{parameters}
define $\mu$ and $\sigma$, and the \stan{guide} block uses those
parameters to propose normal distributions for the model parameters.

\paragraph{Compiling Bayesian neural networks.}

To lift neural networks, we use Pyro
\python{random_module}, a primitive that takes a PyTorch network and a dictionary
of distributions and turns the network into a distribution of
networks where each parameter is sampled from the corresponding
distribution.  We treat network parameters as any other random
variables and apply the comprehensive translation from
\Cref{sec:comprehensive_translation}. This translation
initializes parameters with a uniform prior.
\begin{lstpythons}
priors = {}
priors['l1.weight'] = improper_uniform(shape=[nh, nx])
... # priors of the other parameters
lifted_mlp = pyro.random_module('mlp', mlp, priors)()
\end{lstpythons}
Then, the Stan \stan{~} statements in the \stan{model} block are compiled into Pyro \python{observe} statements.
\begin{lstpythons}
mlp_params = dict(lifted_mlp.named_parameters())
observe(normal(0, 1), mlp_params['l1.weight'])
\end{lstpythons}

It is also possible to mix probabilistic parameters and
non-probabilistic parameters. Our translation only lifts the parameters that are
declared in the \stan{parameters} block by only adding those to the
\python{priors} dictionary.

\section{Evaluation}
\label{sec:evaluation}

We presented our compilation scheme and described the implementation of two new backends for the Stanc3 compiler, targeting Pyro and NumPyro.
This section evaluates our compilation scheme and the proposed extensions.

\subsection{Compiling Stan to Pyro}
First we focus on compiling classic Stan models to Pyro and NumPyro.
We consider three questions:

\begin{description}
\item[RQ1:] Can we compile and run all Stan models?
\item[RQ2:] What is the impact of the compilation on accuracy?
\item[RQ3:] What is the impact of the compilation on speed?
\end{description}

\noindent
To answer these, we used two publicly available benchmark suites: the \texttt{example-models}\footnote{\url{https://github.com/stan-dev/example-models}}
 repository and PosteriorDB~\cite{posteriordb},
a database of Stan models with corresponding data, reference posterior samples, and the configuration used to obtain these samples with Stan.
The experiments were run on a Linux server with 64~cores (2.10GHz, 40GB RAM) with the latest version of Pyro~(1.5.0), NumPyro~(0.4.1), and cmdstanpy~(0.9.67), without GPUs.
The code of the experiments is available at \url{https://github.com/deepppl/evaluation}.

\paragraph{RQ1: Generality of the compilation.}
We run our compiler on the~$541$ models of the \texttt{example-models} repository.
Stanc3 semantics checks reject 10 models.
Out of the remaining~$531$, we were able to compile $522$ models with the comprehensive and mixed compilation schemes for both the Pyro and NumPyro backends, but only~$166$ with the generative scheme.
This further validates the need of our comprehensive translation.
The~$9$ failures all involve truncations, a feature that is not natively supported in Pyro.

To test the inference, we run 1 iteration on the 98 pairs (models, data) of PosteriorDB that can be compiled with Stanc3.
\Cref{tab:eval_compile} presents results for the three compilation schemes: comprehensive, mixed, and generative.

The mixed optimization has no influence on the results.
Failures with the Pyro backend are caused by missing standard library functions that are complicated to port to Pyro.
As discussed in \Cref{sec:implem}, the NumPyro backend relies on JAX, which limits what can be expressed in the model.
The additional errors all involve dynamic features that are not supported in JAX.

As a baseline, we run the same experiment with the generative translation. 
As expected, compilation fails on 60 models. 
The additional runtime errors are the same as for the comprehensive and mixed translations.

\begin{table}
\caption{Successful inference run for 98 PosteriorDB models.}
\label{tab:eval_compile}
\begin{small}
\vspace{-1em}
\begin{tabular}{@{}l rrr@{}}
     & \textsc{Compr.} & \textsc{Mixed} & \textsc{Gener.}\\
\toprule
\textsc{Pyro} & 87 & 87 & 36\\
\textsc{NumPyro} & 83 & 83 & 35\\
\bottomrule
\end{tabular}
\end{small}
\vspace{-1em}
\end{table}

\begin{table*}
  \caption{Comparing inference results with PosteriorDB references}
  \label{tab:eval_accuracy}
\begin{small}
  \vspace{-1em}
  \begin{tabular}{@{}lll@{ }rl@{ }rl@{ }rl@{ }rl@{ }lr@{}}
 
&&&&\multicolumn{2}{c}{\textsc{Pyro}}&\multicolumn{6}{c}{\textsc{NumPyro}}
\\
\cmidrule(lr){5-6}
\cmidrule(lr){7-12}
\textsc{Model} & \textsc{Dataset} & \multicolumn{2}{c}{\textsc{Stan}} &\multicolumn{2}{c}{\textsc{Compr.}} & \multicolumn{2}{c}{\textsc{Compr.}} & \multicolumn{2}{c}{\textsc{Mixed}} & \multicolumn{2}{c}{\textsc{Gener.}} & \textsc{~~~Speedup}\\
    \toprule
    accel\_gp &            mcycle\_gp &     \smark &      00:18:22 &                     \emark &                          &                        \emark &                             &                \emark &                     &                     \emark &                          &      \\
    arK &                  arK &     \smark &      00:00:57 &                   \smark &                    45:39:30 &                      \smark &                       00:00:38 &              \smark &               00:00:37 &                   \smark &                    00:00:34 &    1.48 \\
 arma11 &                 arma &     \smark &      00:01:36 &                   \smark &                    02:19:07 &                      \smark &                       00:00:42 &              \smark &               00:21:46 &                   \smark &                    00:18:53 &    2.26 \\
   dogs &                 dogs &     \smark &      00:01:06 &                   \smark &                    29:12:04 &                      \smark &                       00:06:22 &              \smark &               00:06:12 &                   \smark &                    00:06:09 &    0.17 \\
dogs\_log &                 dogs &     \smark &      00:00:32 &                   \smark &                    23:56:20 &                      \smark &                       00:03:43 &              \smark &               00:03:34 &                     \emark &                          &    0.14 \\
earn\_height &             earnings &     \smark &      00:01:18 &                   \smark &                    01:07:03 &                      \smark &                       00:00:15 &              \smark &               00:00:15 &                     \emark &                          &    5.04 \\
eight\_schools\_centered &        eight\_schools &     \smark &      00:00:05 &                   \smark &                    00:27:09 &                      \smark &                       00:00:07 &              \smark &               00:00:06 &                   \smark &                    00:00:06 &    0.69 \\
eight\_schools\_noncentered &        eight\_schools &     \smark &      00:00:01 &                   \smark &                    00:17:36 &                      \smark &                       00:00:06 &              \smark &               00:00:06 &                  \mmark &                    00:00:06 &    0.19 \\
garch11 &                garch &     \smark &      00:00:17 &                  \mmark &                    20:20:05 &                     \mmark &                       00:02:07 &             \mmark &               00:02:04 &                     \emark &                          &      \\
gp\_regr &         gp\_pois\_regr &     \smark &      00:00:02 &                     \emark &                          &                        \emark &                             &                \emark &                     &                     \emark &                          &      \\
hmm\_drive\_0 &  bball\_drive\_event\_0 &     \smark &      00:03:50 &                   \smark &                   108:34:15 &                      \smark &                       00:25:42 &              \smark &               00:25:38 &                     \emark &                          &    0.15 \\
hmm\_example &          hmm\_example &     \smark &      00:00:28 &                   \smark &                    08:57:44 &                      \smark &                       00:01:02 &              \smark &               00:01:02 &                     \emark &                          &    0.46 \\
kidscore\_interaction &                kidiq &     \smark &      00:01:42 &                   \smark &                    01:40:32 &                      \smark &                       00:00:13 &              \smark &               00:00:13 &                     \emark &                          &    7.80 \\
kidscore\_interaction\_c2 &  kidiq\_with\_mom\_work &     \smark &      00:00:10 &                   \smark &                    00:08:58 &                      \smark &                       00:00:06 &              \smark &               00:00:06 &                     \emark &                          &    1.62 \\
kidscore\_mom\_work &  kidiq\_with\_mom\_work &     \smark &      00:00:14 &                   \smark &                    00:12:16 &                      \smark &                       00:00:09 &              \smark &               00:00:09 &                     \emark &                          &    1.66 \\
kidscore\_momhs &                kidiq &     \smark &      00:00:05 &                   \smark &                    00:12:05 &                      \smark &                       00:00:06 &              \smark &               00:00:06 &                     \emark &                          &    0.86 \\
kidscore\_momhsiq &                kidiq &     \smark &      00:00:28 &                   \smark &                    00:42:54 &                      \smark &                       00:00:08 &              \smark &               00:00:08 &                     \emark &                          &    3.32 \\
kidscore\_momiq &                kidiq &     \smark &      00:00:13 &                   \smark &                    00:30:28 &                      \smark &                       00:00:07 &              \smark &               00:00:07 &                     \emark &                          &    1.82 \\
kilpisjarvi &      kilpisjarvi\_mod &     \smark &      00:00:59 &                   \smark &                    12:12:26 &                      \smark &                       00:00:21 &              \smark &               00:00:21 &                     \emark &                          &    2.87 \\
logearn\_height &             earnings &     \smark &      00:01:19 &                   \smark &                    00:59:53 &                      \smark &                       00:00:15 &              \smark &               00:00:15 &                     \emark &                          &    5.29 \\
logearn\_height\_male &             earnings &     \smark &      00:03:45 &                   \smark &                    01:36:57 &                      \smark &                       00:00:23 &              \smark &               00:00:23 &                     \emark &                          &    9.81 \\
logearn\_logheight\_male &             earnings &     \smark &      00:14:27 &                   \smark &                    06:19:24 &                      \smark &                       00:01:15 &              \smark &               00:01:15 &                     \emark &                          &   11.59 \\
logmesquite\_logvas &             mesquite &     \smark &      00:00:14 &                   \smark &                    00:52:51 &                      \smark &                       00:00:08 &              \smark &               00:00:08 &                     \emark &                          &    1.84 \\
lotka\_volterra &     hudson\_lynx\_hare &     \smark &      00:03:06 &                     \emark &                          &                        \emark &                             &                \emark &                     &                     \emark &                          &      \\
mesquite &             mesquite &     \smark &      00:00:15 &                   \smark &                    00:59:55 &                      \smark &                       00:00:08 &              \smark &               00:00:08 &                     \emark &                          &    1.90 \\
    nes &              nes1980 &     \smark &      00:03:36 &                   \smark &                    00:50:02 &                      \smark &                       00:00:15 &              \smark &               00:00:15 &                     \emark &                          &   14.02 \\
    nes &              nes1976 &     \smark &      00:06:59 &                   \smark &                    00:54:46 &                      \smark &                       00:00:20 &              \smark &               00:00:21 &                     \emark &                          &   20.56 \\
    nes &              nes1972 &     \smark &      00:07:58 &                   \smark &                    00:50:51 &                      \smark &                       00:00:25 &              \smark &               00:00:25 &                     \emark &                          &   19.07 \\
    nes &              nes2000 &     \smark &      00:02:41 &                   \smark &                    00:56:39 &                      \smark &                       00:00:13 &              \smark &               00:00:13 &                     \emark &                          &   12.14 \\
    nes &              nes1996 &     \smark &      00:06:38 &                   \smark &                    00:55:58 &                      \smark &                       00:00:22 &              \smark &               00:00:22 &                     \emark &                          &   18.35 \\
one\_comp\_mm\_elim\_abs & one\_comp\_mm\_elim\_abs &     \smark &      00:16:10 &                     \emark &                          &                        \emark &                             &                \emark &                     &                     \emark &                          &      \\
\bottomrule
\end{tabular}
\\[1em]
    \smark~match, \mmark~mismatch, \emark~error. Durations are reported in \textsc{hh:mm:ss} format.
    \textsc{Speedup} = \textsc{Stan} / \textsc{NumPyro Compr}.
\end{small}
\end{table*}

\paragraph{RQ2: Accuracy.}
To evaluate inference accuracy we compare posterior distributions with the criteria used by regression tests for Stan:\footnote{\url{https://github.com/stan-dev/performance-tests-cmdstan}}
For each parameter, we check if the error between the means is less than $30\%$ of the standard deviation of the reference.
For multidimensional parameters we check the same property for every component: 
\begin{small}
  \begin{center}
$
  | \textrm{mean}(\theta_{\textrm{ref}}) - \textrm{mean}(\theta)| < 0.3 \, \textrm{stddev}(\theta_{\textrm{ref}}).
$
  \end{center}
\end{small}

PosteriorDB provides reference samples for $49$ pairs (model, dataset).  Using Stan with the same configuration (iterations, warmups, chains, thinning, seed), only~$31$ pairs pass the accuracy test and are thus valid baselines for our evaluation.
We run using the Pyro and NumPyro implementation of NUTS with the same configuration\footnote{The configuration interface for NUTS in Pyro and NumPyro is similar to  the CmdStanPy interface.}, and compare the results with the reference posteriors (NUTS, the No U-Turn Sampler~\cite{homan2014nuts}, is an optimized HMC and Stan's preferred inference method). 
\Cref{tab:eval_accuracy} summarizes the results (additional results on the 49 examples are given in \ifextended \Cref{apx:evaluation}\else \cite{deepstan-extended}\fi).

Most of the models yield posterior distributions that match the reference.
The four remaining errors are due to the following missing functions in our implementation of the standard library: \python{cov_exp_quad} (accel\_gp and gp\_regr) and ODE solvers (lotka\_volterra and one\_comp\_mm\_elim\_abs).
The mismatch (garch11) is due to a constraint that we do not know how to compile in Pyro/NumPyro (the domain of a parameter is constrained by the value of another one).

\Cref{tab:eval_accuracy} shows that NUTS in Pyro is much slower than its NumPyro counterpart on our examples.
Due to the high computational cost of running inference in Pyro (more than 100h for hmm\_drive\_0), we focus on the NumPyro backend to compare the different compilation schemes.
Results show no difference between the comprehensive compilation scheme and the mixed version.
When the generative translation is possible, the results also match except for one example (eight\_schools\_noncentered) due to a parameter constraint that is not propagated in the model~(see \Cref{sec:implem}).

As discussed in \Cref{sec:implem}, the mixed compilation scheme can recover the code produced by the generative translation when possible.
\Cref{tab:eval_accuracy} shows that the extra priors introduced by the translation have no impact on the accuracy of the inference when using NUTS.
However, these priors could play a critical role for other inference schemes, e.g., the importance sampling algorithm.

\paragraph{RQ3: Speed.}
To compare inference speed, \Cref{tab:eval_accuracy} reports average runtime for Stan and NumPyro over five runs with varying seed values.
For obvious computational cost reasons, we only report the duration of one Pyro run.
As in RQ2, iterations, warmups, chains, and thinning configurations are given in PosteriorDB.

\Cref{tab:eval_accuracy} shows that the runtime of DeepStan with the NumPyro backend is competitive with Stan under all three compilation schemes.
In addition, runtime durations for the models compiled with the mixed, comprehensive, and generative scheme are almost identical when inference succeeds.
These results indicate that the chosen compilation scheme has negligible influence on inference speed. 
Moreover, as shown in the last column, the NumPyro backend speeds up most benchmarks compared to the highly optimized Stan inference engine (geometric mean: \speedup x on \nbbench benchmarks).

Two examples (eight\_schools-eight\_schools\_centered and arma11)  are very sensitive to random seed variations in Stan and NumPyro (relative standard deviation ${\mathit{std}/\mathit{mean} > 1}$). For all other examples $\mathit{std}/\mathit{mean} \leq 0.1$.

Some of the benchmarks involve nested loops that are still experimental in NumPyro (arK, dogs, dogs\_log, hmm\_drive\_0, hmm\_example).
For these example NumPyro is typically slower that Stan.
If we exclude them we get an overall speedup of 3.8x on 21 benchmarks.

For \Cref{tab:eval_accuracy} we pre-compiled the models to focus on inference time.
The average compilation time for our compiler with both backends was 0.3s (std: 0.02) compared to 10.5s (std: 5.1) for Stan.
The NumPyro backend thus outperforms Stan in both compilation and inference speed for these examples.

\subsection{Stan Extensions}

This section evaluates DeepStan, our extension with explicit variational guides and support for deep probabilistic programming with neural networks.
We consider two questions:

\pagebreak

\begin{description}
  \item[RQ4:] Are explicit variational guides useful?
  \item[RQ5:] For deep probabilistic models, how does DeepStan compare to hand-written Pyro code?
\end{description}

\paragraph{RQ4: Explicit guides.}
The \emph{multimodal} example shown in \cref{fig:multimodal} is a mixture of two Gaussian distributions with different means but identical variances.
The first two histograms of \cref{fig:multimodal} show that in both Stan and DeepStan, this example is particularly challenging for NUTS.
Using multiple chains, NUTS finds the two modes, but the chains do not mix and the relative densities are incorrect.
This is a known limitation of HMC.\footnote{\url{https://mc-stan.org/users/documentation/case-studies/identifying_mixture_models.html}}

Stan also offers ADVI~\cite{advi17jmlr}, an implementation of black-box VI where guides are automatically synthesized from the model using a mean-field approximation.
This choice implies that ADVI cannot approximate multi-modal distribution as illustrated in the last histogram of \cref{fig:multimodal}.
On the other hand, using the custom variational guide presented in \Cref{fig:multimodal}, DeepStan with VI is able to find the two modes.

\begin{figure}

\begin{minipage}[t]{0.51\linewidth}
\vspace{-3.35em}
\begin{lststantable}
parameters {
  real cluster; real theta; }
model {
  real mu;
  cluster ~ normal(0, 1);
  if (cluster > 0) mu = 20;
  else mu = 0;
  theta ~ normal(mu, 1); }
guide parameters {
  real m1; real m2;
  real<lower=0> s1;
  real<lower=0> s2; }
guide {
  cluster ~ normal(0, 1);
  if (cluster > 0) theta ~ normal(m1, s1);
  else theta ~ normal(m2, s2); }
\end{lststantable}
\end{minipage}
\begin{minipage}[t]{0.48\linewidth}
\includegraphics[scale=0.247, trim={.5cm 3.07cm 0.5cm 2.5cm}, clip]{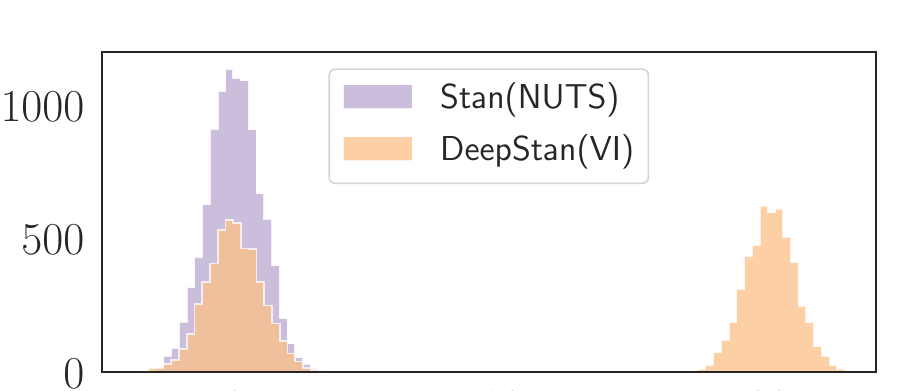}
  \includegraphics[scale=0.247, trim={.5cm 3.07cm 0.5cm 2.5cm}, clip]{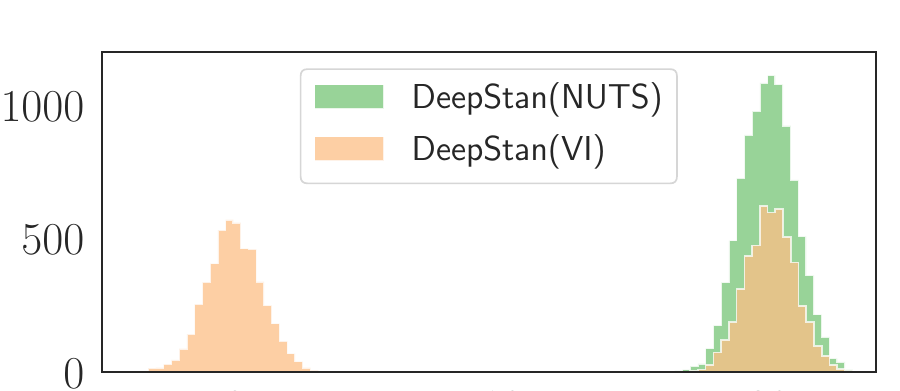}
  \includegraphics[scale=0.247, trim={.5cm 1.3cm 0.5cm 2.5cm}, clip]{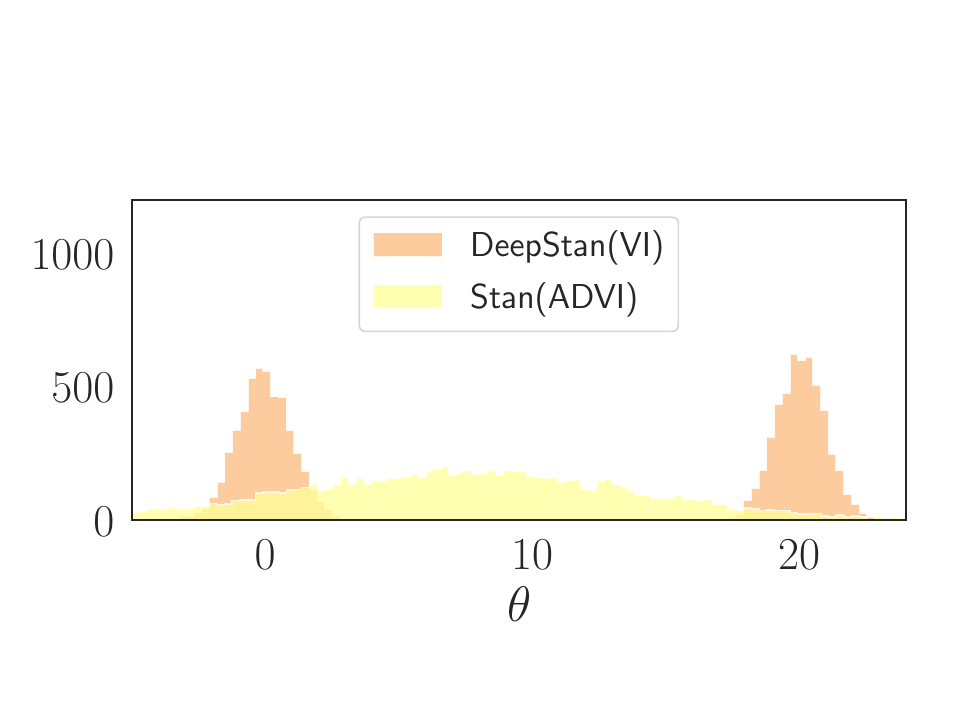}
\end{minipage}
\vspace*{-0.4em}
\caption{\label{fig:multimodal}
DeepStan code and histograms of the multimodal example using
Stan, DeepStan with NUTS, DeepStan with VI, and Stan with ADVI.}
\vspace*{-0.25em}
\end{figure}

\paragraph{RQ5: Deep probabilistic models.}\label{sec:exp_deep}
As Stan lacks support for deep probabilistic models, it cannot be used as a baseline.
Instead, we compare the performance of the compiled code with hand-written Pyro code on the VAE described in \Cref{sec:nn} and a simple Bayesian neural network.

Variational autoencoders were not designed as a predictive model but as a generative model to reconstruct images.
Evaluating the performance of a VAE is thus non-obvious.
We trained two VAEs on the MNIST dataset using VI:
one hand-written in Pyro, the other written in DeepStan.
For each image in the test set, the trained VAEs compute a latent representation of dimension~5.
We cluster these representations using KMeans with~10 clusters.
We measure VAE performance with the pairwise F1 metric:
true positives are the number of images of the same digit that appear in the same cluster.
For Pyro F1=0.41 (precision=0.43, recall=0.40), and for DeepStan F1=0.43 (precision=0.44, recall=0.42).
These numbers shows that compiling DeepStan to Pyro does not impact the performance of such deep probabilistic models.

We trained two implementations of a 2-level Bayesian multi-layer perceptron~(MLP) with the parameters all lifted to random variables (see \cref{sec:mlp}):
one hand-written in Pyro, the other written in DeepStan.
We trained both models for~20 epochs on the training set.
For each model we generated~100 samples of concrete weights and biases to obtain an ensemble MLP that can be used to compute a distribution of predicted labels.
The accuracy for both models is 92\% on the test set and the agreement between the two models is above 95\%.
The execution time is comparable.
These experiments show that compiling DeepStan models to Pyro has little impact on the model.
Changing the priors on the network parameters from \stan{normal(0,1)} to \stan{normal(0,10)} (see \Cref{sec:mlp}) increases accuracy from~0.92 to~0.96.
This further validates our compilation, compiling parameter priors to \python{observe} statements on deep probabilistic models.

 \section{Related work}\label{sec:related}

To the best of our knowledge, we propose the first comprehensive translation of Stan to a generative PPL.
The closest related work was developed by the Pyro team~\cite{chen2018transpiling}.
Their work focuses on performance and ours on completeness.
Their proposed compilation technique corresponds to the generative translation presented in \Cref{sec:generative_translation} and thus only handles a subset of Stan.
The code is not open-source, and we rely on our own implementation of the generative translation in \Cref{sec:evaluation}.
Compared to our approach, they are also looking into independence assumptions between loop iterations to generate parallel code.
Combining these ideas with our approach
is a promising future direction.
They do not extend Stan with either VI or neural networks.
Similarly, in Appendix B.2 of~\cite{GorinovaGS19}, Gorinova et al.\ outline the generative translation of \Cref{sec:generative_translation}, and also mention the issue with multiple updates but do not provide a solution.
\citet{Wonyeol_et_al_2020} introduce a density-based semantics for Pyro, but this semantics does not handle Stan's non-generative features.

The goal of compiling Stan to Pyro is to create a platform for experimenting with new ideas.
For example, \Cref{sec:guides} extends Stan with explicit variational guides.
Similarly, Pyro now offers inference on discrete parameters\footnote{https://pyro.ai/examples/enumeration.html} that we could port to Stan using our backends.

In recent years, taking advantage of the maturity of DL
frameworks, multiple deep probabilistic programming languages have
been proposed:
PyMC3~\cite{salvatier_wiecki_fonnesbeck_2015} built on top of Theano,
Edward~\cite{tran_et_al_2017} and
ZhuSuan~\cite{shi_et_al_2017} built on top of TensorFlow, and
Pyro~\cite{bingham_et_al_2019} and
ProbTorch~\cite{siddharth_et_al_2017} built on top of PyTorch.
All these languages are implemented as libraries. The users
thus need to master the entire technology stack of the library, the
underlying DL framework, and the host language.
In comparison, DeepStan is a self-contained language and the compiler helps the programmer via dedicated static analyses.

 \section{Conclusion}\label{sec:conclusion}

This paper introduces a comprehensive compilation scheme from Stan to generative probabilistic programming languages.
This shows that Stan is at most as expressive as this family of languages.
We implemented a compiler from Stan to Pyro.
Additionally, we designed and implemented extensions for Stan with
explicit variational guides and neural networks.

\paragraph{Acknowledgement}
The authors are greatful to the following people for their helpful feedback and encouragements during this work: E. Bingham, K. Kate, Y. Mroueh, F. Obermeyer, and A. Pauthier.

\balance
\bibliography{bibfile}


\begin{thebibliography}{35}


\ifx \showCODEN    \undefined \def \showCODEN     #1{\unskip}     \fi
\ifx \showDOI      \undefined \def \showDOI       #1{#1}\fi
\ifx \showISBNx    \undefined \def \showISBNx     #1{\unskip}     \fi
\ifx \showISBNxiii \undefined \def \showISBNxiii  #1{\unskip}     \fi
\ifx \showISSN     \undefined \def \showISSN      #1{\unskip}     \fi
\ifx \showLCCN     \undefined \def \showLCCN      #1{\unskip}     \fi
\ifx \shownote     \undefined \def \shownote      #1{#1}          \fi
\ifx \showarticletitle \undefined \def \showarticletitle #1{#1}   \fi
\ifx \showURL      \undefined \def \showURL       {\relax}        \fi
\providecommand\bibfield[2]{#2}
\providecommand\bibinfo[2]{#2}
\providecommand\natexlab[1]{#1}
\providecommand\showeprint[2][]{arXiv:#2}

\bibitem[\protect\citeauthoryear{Baudart, Burroni, Hirzel, Mandel, and
  Shinnar}{Baudart et~al\mbox{.}}{2021}]%
        {deepstan-short}
\bibfield{author}{\bibinfo{person}{Guillaume Baudart}, \bibinfo{person}{Javier
  Burroni}, \bibinfo{person}{Martin Hirzel}, \bibinfo{person}{Louis Mandel},
  {and} \bibinfo{person}{Avraham Shinnar}.} \bibinfo{year}{2021}\natexlab{}.
\newblock \showarticletitle{Compiling Stan to Generative Probabilistic
  Languages and Extension to Deep Probabilistic Programming}. In
  \bibinfo{booktitle}{\emph{{PLDI}}}. {ACM}.
\newblock


\bibitem[\protect\citeauthoryear{Baudart, Hirzel, and Mandel}{Baudart
  et~al\mbox{.}}{2018}]%
        {baudart_hirzel_mandel_2018}
\bibfield{author}{\bibinfo{person}{Guillaume Baudart}, \bibinfo{person}{Martin
  Hirzel}, {and} \bibinfo{person}{Louis Mandel}.}
  \bibinfo{year}{2018}\natexlab{}.
\newblock \showarticletitle{Deep Probabilistic Programming Languages: {A}
  Qualitative Study}.
\newblock \bibinfo{journal}{\emph{CoRR}}  \bibinfo{volume}{abs/1804.06458}
  (\bibinfo{year}{2018}).
\newblock


\bibitem[\protect\citeauthoryear{Bingham, Chen, Jankowiak, Obermeyer, Pradhan,
  Karaletsos, Singh, Szerlip, Horsfall, and Goodman}{Bingham
  et~al\mbox{.}}{2019}]%
        {bingham_et_al_2019}
\bibfield{author}{\bibinfo{person}{Eli Bingham}, \bibinfo{person}{Jonathan~P.
  Chen}, \bibinfo{person}{Martin Jankowiak}, \bibinfo{person}{Fritz Obermeyer},
  \bibinfo{person}{Neeraj Pradhan}, \bibinfo{person}{Theofanis Karaletsos},
  \bibinfo{person}{Rohit Singh}, \bibinfo{person}{Paul~A. Szerlip},
  \bibinfo{person}{Paul Horsfall}, {and} \bibinfo{person}{Noah~D. Goodman}.}
  \bibinfo{year}{2019}\natexlab{}.
\newblock \showarticletitle{Pyro: Deep Universal Probabilistic Programming}.
\newblock \bibinfo{journal}{\emph{J. Mach. Learn. Res.}}  \bibinfo{volume}{20}
  (\bibinfo{year}{2019}), \bibinfo{pages}{28:1--28:6}.
\newblock


\bibitem[\protect\citeauthoryear{Blei, Kucukelbir, and McAuliffe}{Blei
  et~al\mbox{.}}{2016}]%
        {blei_kucukelbir_mcauliffe_2017}
\bibfield{author}{\bibinfo{person}{David~M. Blei}, \bibinfo{person}{Alp
  Kucukelbir}, {and} \bibinfo{person}{Jon~D. McAuliffe}.}
  \bibinfo{year}{2016}\natexlab{}.
\newblock \showarticletitle{Variational Inference: {A} Review for
  Statisticians}.
\newblock \bibinfo{journal}{\emph{CoRR}}  \bibinfo{volume}{abs/1601.00670}
  (\bibinfo{year}{2016}).
\newblock


\bibitem[\protect\citeauthoryear{Bradbury, Frostig, Hawkins, Johnson, Leary,
  Maclaurin, and Wanderman-Milne}{Bradbury et~al\mbox{.}}{2018}]%
        {jax2018github}
\bibfield{author}{\bibinfo{person}{James Bradbury}, \bibinfo{person}{Roy
  Frostig}, \bibinfo{person}{Peter Hawkins}, \bibinfo{person}{Matthew~James
  Johnson}, \bibinfo{person}{Chris Leary}, \bibinfo{person}{Dougal Maclaurin},
  {and} \bibinfo{person}{Skye Wanderman-Milne}.}
  \bibinfo{year}{2018}\natexlab{}.
\newblock \bibinfo{booktitle}{\emph{{JAX}: composable transformations of
  {P}ython+{N}um{P}y programs}}.
\newblock
\urldef\tempurl%
\url{http://github.com/google/jax}
\showURL{%
\tempurl}


\bibitem[\protect\citeauthoryear{Carlin and Louis}{Carlin and Louis}{2008}]%
        {carlin2008bayesian}
\bibfield{author}{\bibinfo{person}{Bradley~P Carlin} {and}
  \bibinfo{person}{Thomas~A Louis}.} \bibinfo{year}{2008}\natexlab{}.
\newblock \bibinfo{booktitle}{\emph{Bayesian methods for data analysis}}.
\newblock \bibinfo{publisher}{CRC Press}.
\newblock


\bibitem[\protect\citeauthoryear{Carpenter, Gelman, Hoffman, Lee, Goodrich,
  Betancourt, Brubaker, Guo, Li, and Riddell}{Carpenter et~al\mbox{.}}{2017}]%
        {carpenter2017stan}
\bibfield{author}{\bibinfo{person}{Bob Carpenter}, \bibinfo{person}{Andrew
  Gelman}, \bibinfo{person}{Matthew~D Hoffman}, \bibinfo{person}{Daniel Lee},
  \bibinfo{person}{Ben Goodrich}, \bibinfo{person}{Michael Betancourt},
  \bibinfo{person}{Marcus Brubaker}, \bibinfo{person}{Jiqiang Guo},
  \bibinfo{person}{Peter Li}, {and} \bibinfo{person}{Allen Riddell}.}
  \bibinfo{year}{2017}\natexlab{}.
\newblock \showarticletitle{Stan: A probabilistic programming language}.
\newblock \bibinfo{journal}{\emph{Journal of Statistical Software}}
  \bibinfo{volume}{76}, \bibinfo{number}{1} (\bibinfo{year}{2017}),
  \bibinfo{pages}{1--37}.
\newblock
\urldef\tempurl%
\url{https://doi.org/10.18637/jss.v076.i01}
\showDOI{\tempurl}


\bibitem[\protect\citeauthoryear{Chen, Singh, Bingham, and Goodman}{Chen
  et~al\mbox{.}}{2018}]%
        {chen2018transpiling}
\bibfield{author}{\bibinfo{person}{Jonathan~P. Chen}, \bibinfo{person}{Rohit
  Singh}, \bibinfo{person}{Eli Bingham}, {and} \bibinfo{person}{Noah Goodman}.}
  \bibinfo{year}{2018}\natexlab{}.
\newblock \showarticletitle{Transpiling {Stan} models to {Pyro}}. In
  \bibinfo{booktitle}{\emph{{ProbProg}}}.
\newblock


\bibitem[\protect\citeauthoryear{Cusumano{-}Towner, Saad, Lew, and
  Mansinghka}{Cusumano{-}Towner et~al\mbox{.}}{2019}]%
        {cusamotowner_et_al_2019}
\bibfield{author}{\bibinfo{person}{Marco~F. Cusumano{-}Towner},
  \bibinfo{person}{Feras~A. Saad}, \bibinfo{person}{Alexander~K. Lew}, {and}
  \bibinfo{person}{Vikash~K. Mansinghka}.} \bibinfo{year}{2019}\natexlab{}.
\newblock \showarticletitle{Gen: a general-purpose probabilistic programming
  system with programmable inference}. In \bibinfo{booktitle}{\emph{{PLDI}}}.
  \bibinfo{publisher}{{ACM}}, \bibinfo{pages}{221--236}.
\newblock
\urldef\tempurl%
\url{https://doi.org/10.1145/3314221.3314642}
\showDOI{\tempurl}


\bibitem[\protect\citeauthoryear{Gelman and Hill}{Gelman and Hill}{2006}]%
        {gelman2006data}
\bibfield{author}{\bibinfo{person}{Andrew Gelman} {and}
  \bibinfo{person}{Jennifer Hill}.} \bibinfo{year}{2006}\natexlab{}.
\newblock \bibinfo{booktitle}{\emph{Data analysis using regression and
  multilevel/hierarchical models}}.
\newblock \bibinfo{publisher}{Cambridge university press}.
\newblock
\urldef\tempurl%
\url{https://doi.org/10.1017/CBO9780511790942}
\showDOI{\tempurl}


\bibitem[\protect\citeauthoryear{Gelman, Stern, Carlin, Dunson, Vehtari, and
  Rubin}{Gelman et~al\mbox{.}}{2013}]%
        {gelman2013bayesian}
\bibfield{author}{\bibinfo{person}{Andrew Gelman}, \bibinfo{person}{Hal~S
  Stern}, \bibinfo{person}{John~B Carlin}, \bibinfo{person}{David~B Dunson},
  \bibinfo{person}{Aki Vehtari}, {and} \bibinfo{person}{Donald~B Rubin}.}
  \bibinfo{year}{2013}\natexlab{}.
\newblock \bibinfo{booktitle}{\emph{Bayesian data analysis}}.
\newblock \bibinfo{publisher}{Chapman and Hall/CRC}.
\newblock


\bibitem[\protect\citeauthoryear{Goodman, Mansinghka, Roy, Bonawitz, and
  Tenenbaum}{Goodman et~al\mbox{.}}{2008}]%
        {goodman_et_al_2008}
\bibfield{author}{\bibinfo{person}{Noah~D. Goodman}, \bibinfo{person}{Vikash~K.
  Mansinghka}, \bibinfo{person}{Daniel~M. Roy}, \bibinfo{person}{Keith
  Bonawitz}, {and} \bibinfo{person}{Joshua~B. Tenenbaum}.}
  \bibinfo{year}{2008}\natexlab{}.
\newblock \showarticletitle{Church: a language for generative models}. In
  \bibinfo{booktitle}{\emph{{UAI}}}. \bibinfo{publisher}{{AUAI} Press},
  \bibinfo{pages}{220--229}.
\newblock


\bibitem[\protect\citeauthoryear{Goodman and Stuhlm{\"u}ller}{Goodman and
  Stuhlm{\"u}ller}{2014}]%
        {goodman_stuhlmuller_2014}
\bibfield{author}{\bibinfo{person}{Noah~D. Goodman} {and}
  \bibinfo{person}{Andreas Stuhlm{\"u}ller}.} \bibinfo{year}{2014}\natexlab{}.
\newblock \bibinfo{title}{The Design and Implementation of Probabilistic
  Programming Languages}.
\newblock
\newblock
\urldef\tempurl%
\url{http://dippl.org}
\showURL{%
\tempurl}
\newblock
\shownote{Accessed April 2021.}


\bibitem[\protect\citeauthoryear{Gordon, Henzinger, Nori, and Rajamani}{Gordon
  et~al\mbox{.}}{2014}]%
        {gordon_et_al_2014}
\bibfield{author}{\bibinfo{person}{Andrew~D. Gordon},
  \bibinfo{person}{Thomas~A. Henzinger}, \bibinfo{person}{Aditya~V. Nori},
  {and} \bibinfo{person}{Sriram~K. Rajamani}.} \bibinfo{year}{2014}\natexlab{}.
\newblock \showarticletitle{Probabilistic programming}. In
  \bibinfo{booktitle}{\emph{{FOSE}}}. \bibinfo{publisher}{{ACM}},
  \bibinfo{pages}{167--181}.
\newblock
\urldef\tempurl%
\url{https://doi.org/10.1145/2593882.2593900}
\showDOI{\tempurl}


\bibitem[\protect\citeauthoryear{Gorinova, Gordon, and Sutton}{Gorinova
  et~al\mbox{.}}{2019}]%
        {GorinovaGS19}
\bibfield{author}{\bibinfo{person}{Maria~I. Gorinova},
  \bibinfo{person}{Andrew~D. Gordon}, {and} \bibinfo{person}{Charles Sutton}.}
  \bibinfo{year}{2019}\natexlab{}.
\newblock \showarticletitle{Probabilistic programming with densities in
  SlicStan: efficient, flexible, and deterministic}.
\newblock \bibinfo{journal}{\emph{Proc. {ACM} Program. Lang.}}
  \bibinfo{volume}{3}, \bibinfo{number}{{POPL}} (\bibinfo{year}{2019}),
  \bibinfo{pages}{35:1--35:30}.
\newblock
\urldef\tempurl%
\url{https://doi.org/10.1145/3290348}
\showURL{%
\tempurl}


\bibitem[\protect\citeauthoryear{Hoffman and Gelman}{Hoffman and
  Gelman}{2014}]%
        {homan2014nuts}
\bibfield{author}{\bibinfo{person}{Matthew~D. Hoffman} {and}
  \bibinfo{person}{Andrew Gelman}.} \bibinfo{year}{2014}\natexlab{}.
\newblock \showarticletitle{The No-U-turn sampler: adaptively setting path
  lengths in Hamiltonian Monte Carlo}.
\newblock \bibinfo{journal}{\emph{J. Mach. Learn. Res.}} \bibinfo{volume}{15},
  \bibinfo{number}{1} (\bibinfo{year}{2014}), \bibinfo{pages}{1593--1623}.
\newblock


\bibitem[\protect\citeauthoryear{Kingma and Welling}{Kingma and
  Welling}{2014}]%
        {kingma_welling_2013}
\bibfield{author}{\bibinfo{person}{Diederik~P. Kingma} {and}
  \bibinfo{person}{Max Welling}.} \bibinfo{year}{2014}\natexlab{}.
\newblock \showarticletitle{Auto-Encoding Variational Bayes}. In
  \bibinfo{booktitle}{\emph{{ICLR}}}.
\newblock


\bibitem[\protect\citeauthoryear{Kozen}{Kozen}{1981}]%
        {kozen_1981}
\bibfield{author}{\bibinfo{person}{Dexter Kozen}.}
  \bibinfo{year}{1981}\natexlab{}.
\newblock \showarticletitle{Semantics of Probabilistic Programs}.
\newblock \bibinfo{journal}{\emph{J. Comput. Syst. Sci.}} \bibinfo{volume}{22},
  \bibinfo{number}{3} (\bibinfo{year}{1981}), \bibinfo{pages}{328--350}.
\newblock
\urldef\tempurl%
\url{https://doi.org/10.1016/0022-0000(81)90036-2}
\showDOI{\tempurl}


\bibitem[\protect\citeauthoryear{Kucukelbir, Tran, Ranganath, Gelman, and
  Blei}{Kucukelbir et~al\mbox{.}}{2017}]%
        {advi17jmlr}
\bibfield{author}{\bibinfo{person}{Alp Kucukelbir}, \bibinfo{person}{Dustin
  Tran}, \bibinfo{person}{Rajesh Ranganath}, \bibinfo{person}{Andrew Gelman},
  {and} \bibinfo{person}{David~M. Blei}.} \bibinfo{year}{2017}\natexlab{}.
\newblock \showarticletitle{Automatic Differentiation Variational Inference}.
\newblock \bibinfo{journal}{\emph{J. Mach. Learn. Res.}}  \bibinfo{volume}{18}
  (\bibinfo{year}{2017}), \bibinfo{pages}{14:1--14:45}.
\newblock


\bibitem[\protect\citeauthoryear{Lee, Yu, Rival, and Yang}{Lee
  et~al\mbox{.}}{2020}]%
        {Wonyeol_et_al_2020}
\bibfield{author}{\bibinfo{person}{Wonyeol Lee}, \bibinfo{person}{Hangyeol Yu},
  \bibinfo{person}{Xavier Rival}, {and} \bibinfo{person}{Hongseok Yang}.}
  \bibinfo{year}{2020}\natexlab{}.
\newblock \showarticletitle{Towards verified stochastic variational inference
  for probabilistic programs}.
\newblock \bibinfo{journal}{\emph{{PACMPL}}} \bibinfo{volume}{4},
  \bibinfo{number}{{POPL}} (\bibinfo{year}{2020}),
  \bibinfo{pages}{16:1--16:33}.
\newblock
\urldef\tempurl%
\url{https://doi.org/10.1145/3371084}
\showDOI{\tempurl}


\bibitem[\protect\citeauthoryear{Lunn, Spiegelhalter, Thomas, and Best}{Lunn
  et~al\mbox{.}}{2009}]%
        {lunn2009bugs}
\bibfield{author}{\bibinfo{person}{David Lunn}, \bibinfo{person}{David
  Spiegelhalter}, \bibinfo{person}{Andrew Thomas}, {and} \bibinfo{person}{Nicky
  Best}.} \bibinfo{year}{2009}\natexlab{}.
\newblock \showarticletitle{The BUGS project: Evolution, critique and future
  directions}.
\newblock \bibinfo{journal}{\emph{Stat. in medicine}} \bibinfo{volume}{28},
  \bibinfo{number}{25} (\bibinfo{year}{2009}), \bibinfo{pages}{3049--3067}.
\newblock
\urldef\tempurl%
\url{https://doi.org/10.1002/sim.3680}
\showDOI{\tempurl}


\bibitem[\protect\citeauthoryear{Narayanaswamy, Paige, van~de Meent, Desmaison,
  Goodman, Kohli, Wood, and Torr}{Narayanaswamy et~al\mbox{.}}{2017}]%
        {siddharth_et_al_2017}
\bibfield{author}{\bibinfo{person}{Siddharth Narayanaswamy},
  \bibinfo{person}{Brooks Paige}, \bibinfo{person}{Jan{-}Willem van~de Meent},
  \bibinfo{person}{Alban Desmaison}, \bibinfo{person}{Noah~D. Goodman},
  \bibinfo{person}{Pushmeet Kohli}, \bibinfo{person}{Frank~D. Wood}, {and}
  \bibinfo{person}{Philip H.~S. Torr}.} \bibinfo{year}{2017}\natexlab{}.
\newblock \showarticletitle{Learning Disentangled Representations with
  Semi-Supervised Deep Generative Models}. In
  \bibinfo{booktitle}{\emph{{NIPS}}}. \bibinfo{pages}{5925--5935}.
\newblock


\bibitem[\protect\citeauthoryear{Neal}{Neal}{1996}]%
        {neal_2012}
\bibfield{author}{\bibinfo{person}{Radford~M. Neal}.}
  \bibinfo{year}{1996}\natexlab{}.
\newblock \bibinfo{booktitle}{\emph{Bayesian Learning for Neural Networks}}.
  Vol.~\bibinfo{volume}{118}.
\newblock \bibinfo{publisher}{Springer}.
\newblock
\urldef\tempurl%
\url{https://doi.org/10.1007/978-1-4612-0745-0}
\showDOI{\tempurl}


\bibitem[\protect\citeauthoryear{Paszke, Gross, Chintala, Chanan, Yang, DeVito,
  Lin, Desmaison, Antiga, and Lerer}{Paszke et~al\mbox{.}}{2017}]%
        {paszke_et_al_2017}
\bibfield{author}{\bibinfo{person}{Adam Paszke}, \bibinfo{person}{Sam Gross},
  \bibinfo{person}{Soumith Chintala}, \bibinfo{person}{Gregory Chanan},
  \bibinfo{person}{Edward Yang}, \bibinfo{person}{Zachary DeVito},
  \bibinfo{person}{Zeming Lin}, \bibinfo{person}{Alban Desmaison},
  \bibinfo{person}{Luca Antiga}, {and} \bibinfo{person}{Adam Lerer}.}
  \bibinfo{year}{2017}\natexlab{}.
\newblock \showarticletitle{Automatic Differentiation in {PyTorch}}. In
  \bibinfo{booktitle}{\emph{AutoDiff Workshop}}.
\newblock


\bibitem[\protect\citeauthoryear{Phan, Pradhan, and Jankowiak}{Phan
  et~al\mbox{.}}{2019}]%
        {numpyro_2019}
\bibfield{author}{\bibinfo{person}{Du Phan}, \bibinfo{person}{Neeraj Pradhan},
  {and} \bibinfo{person}{Martin Jankowiak}.} \bibinfo{year}{2019}\natexlab{}.
\newblock \showarticletitle{Composable Effects for Flexible and Accelerated
  Probabilistic Programming in NumPyro}.
\newblock \bibinfo{journal}{\emph{CoRR}}  \bibinfo{volume}{abs/1912.11554}
  (\bibinfo{year}{2019}).
\newblock


\bibitem[\protect\citeauthoryear{Plummer et~al\mbox{.}}{Plummer
  et~al\mbox{.}}{2003}]%
        {plummer2003jags}
\bibfield{author}{\bibinfo{person}{Martyn Plummer} {et~al\mbox{.}}}
  \bibinfo{year}{2003}\natexlab{}.
\newblock \showarticletitle{JAGS: A program for analysis of Bayesian graphical
  models using Gibbs sampling}. In \bibinfo{booktitle}{\emph{Workshop on distr.
  stat. comp.}}, Vol.~\bibinfo{volume}{124}. Vienna, Austria.
\newblock


\bibitem[\protect\citeauthoryear{Rezende, Mohamed, and Wierstra}{Rezende
  et~al\mbox{.}}{2014}]%
        {rezende_mohamed_wierstra_2014}
\bibfield{author}{\bibinfo{person}{Danilo~Jimenez Rezende},
  \bibinfo{person}{Shakir Mohamed}, {and} \bibinfo{person}{Daan Wierstra}.}
  \bibinfo{year}{2014}\natexlab{}.
\newblock \showarticletitle{Stochastic Backpropagation and Approximate
  Inference in Deep Generative Models}. In \bibinfo{booktitle}{\emph{{ICML}}}
  \emph{(\bibinfo{series}{{JMLR} Workshop and Conference Proceedings},
  Vol.~\bibinfo{volume}{32})}. \bibinfo{publisher}{JMLR.org},
  \bibinfo{pages}{1278--1286}.
\newblock


\bibitem[\protect\citeauthoryear{Salvatier, Wiecki, and Fonnesbeck}{Salvatier
  et~al\mbox{.}}{2016}]%
        {salvatier_wiecki_fonnesbeck_2015}
\bibfield{author}{\bibinfo{person}{John Salvatier}, \bibinfo{person}{Thomas~V.
  Wiecki}, {and} \bibinfo{person}{Christopher Fonnesbeck}.}
  \bibinfo{year}{2016}\natexlab{}.
\newblock \showarticletitle{Probabilistic programming in Python using PyMC3}.
\newblock \bibinfo{journal}{\emph{PeerJ Comput. Sci.}}  \bibinfo{volume}{2}
  (\bibinfo{year}{2016}), \bibinfo{pages}{e55}.
\newblock
\urldef\tempurl%
\url{https://doi.org/10.7717/peerj-cs.55}
\showDOI{\tempurl}


\bibitem[\protect\citeauthoryear{Shi, Chen, Zhu, Sun, Luo, Gu, and Zhou}{Shi
  et~al\mbox{.}}{2017}]%
        {shi_et_al_2017}
\bibfield{author}{\bibinfo{person}{Jiaxin Shi}, \bibinfo{person}{Jianfei Chen},
  \bibinfo{person}{Jun Zhu}, \bibinfo{person}{Shengyang Sun},
  \bibinfo{person}{Yucen Luo}, \bibinfo{person}{Yihong Gu}, {and}
  \bibinfo{person}{Yuhao Zhou}.} \bibinfo{year}{2017}\natexlab{}.
\newblock \showarticletitle{ZhuSuan: {A} Library for Bayesian Deep Learning}.
\newblock \bibinfo{journal}{\emph{CoRR}}  \bibinfo{volume}{abs/1709.05870}
  (\bibinfo{year}{2017}).
\newblock


\bibitem[\protect\citeauthoryear{Staton}{Staton}{2017}]%
        {staton_2017}
\bibfield{author}{\bibinfo{person}{Sam Staton}.}
  \bibinfo{year}{2017}\natexlab{}.
\newblock \showarticletitle{Commutative Semantics for Probabilistic
  Programming}. In \bibinfo{booktitle}{\emph{{ESOP}}}
  \emph{(\bibinfo{series}{Lecture Notes in Computer Science},
  Vol.~\bibinfo{volume}{10201})}. \bibinfo{publisher}{Springer},
  \bibinfo{pages}{855--879}.
\newblock
\urldef\tempurl%
\url{https://doi.org/10.1007/978-3-662-54434-1_32}
\showDOI{\tempurl}


\bibitem[\protect\citeauthoryear{Staton, Yang, Wood, Heunen, and Kammar}{Staton
  et~al\mbox{.}}{2016}]%
        {staton_et_al_2016}
\bibfield{author}{\bibinfo{person}{Sam Staton}, \bibinfo{person}{Hongseok
  Yang}, \bibinfo{person}{Frank~D. Wood}, \bibinfo{person}{Chris Heunen}, {and}
  \bibinfo{person}{Ohad Kammar}.} \bibinfo{year}{2016}\natexlab{}.
\newblock \showarticletitle{Semantics for probabilistic programming:
  higher-order functions, continuous distributions, and soft constraints}. In
  \bibinfo{booktitle}{\emph{{LICS}}}. \bibinfo{publisher}{{ACM}},
  \bibinfo{pages}{525--534}.
\newblock
\urldef\tempurl%
\url{https://doi.org/10.1145/2933575.2935313}
\showDOI{\tempurl}


\bibitem[\protect\citeauthoryear{Tolpin, van~de Meent, Yang, and Wood}{Tolpin
  et~al\mbox{.}}{2016}]%
        {tolpin_et_al_2016}
\bibfield{author}{\bibinfo{person}{David Tolpin}, \bibinfo{person}{Jan{-}Willem
  van~de Meent}, \bibinfo{person}{Hongseok Yang}, {and}
  \bibinfo{person}{Frank~D. Wood}.} \bibinfo{year}{2016}\natexlab{}.
\newblock \showarticletitle{Design and Implementation of Probabilistic
  Programming Language Anglican}. In \bibinfo{booktitle}{\emph{{IFL}}}.
  \bibinfo{publisher}{{ACM}}, \bibinfo{pages}{6:1--6:12}.
\newblock
\urldef\tempurl%
\url{https://doi.org/10.1145/3064899.3064910}
\showDOI{\tempurl}


\bibitem[\protect\citeauthoryear{Tran, Hoffman, Saurous, Brevdo, Murphy, and
  Blei}{Tran et~al\mbox{.}}{2017}]%
        {tran_et_al_2017}
\bibfield{author}{\bibinfo{person}{Dustin Tran}, \bibinfo{person}{Matthew~D.
  Hoffman}, \bibinfo{person}{Rif~A. Saurous}, \bibinfo{person}{Eugene Brevdo},
  \bibinfo{person}{Kevin Murphy}, {and} \bibinfo{person}{David~M. Blei}.}
  \bibinfo{year}{2017}\natexlab{}.
\newblock \showarticletitle{Deep Probabilistic Programming}. In
  \bibinfo{booktitle}{\emph{{ICLR} (Poster)}}.
\newblock


\bibitem[\protect\citeauthoryear{van~de Meent, Paige, Yang, and Wood}{van~de
  Meent et~al\mbox{.}}{2018}]%
        {van2018introduction}
\bibfield{author}{\bibinfo{person}{Jan{-}Willem van~de Meent},
  \bibinfo{person}{Brooks Paige}, \bibinfo{person}{Hongseok Yang}, {and}
  \bibinfo{person}{Frank Wood}.} \bibinfo{year}{2018}\natexlab{}.
\newblock \showarticletitle{An Introduction to Probabilistic Programming}.
\newblock \bibinfo{journal}{\emph{CoRR}}  \bibinfo{volume}{abs/1809.10756}
  (\bibinfo{year}{2018}).
\newblock


\bibitem[\protect\citeauthoryear{Vehtari and Magnusson}{Vehtari and
  Magnusson}{2020}]%
        {posteriordb}
\bibfield{author}{\bibinfo{person}{Aki Vehtari} {and} \bibinfo{person}{Måns
  Magnusson}.} \bibinfo{year}{2020}\natexlab{}.
\newblock \showarticletitle{{PosteriorDB}: a database with data, models and
  posteriors}. In \bibinfo{booktitle}{\emph{Stan Conf.}}
\newblock
\urldef\tempurl%
\url{https://github.com/stan-dev/posteriordb}
\showURL{%
\tempurl}


\end{thebibliography}
\bibliographystyle{acm-reference-format}

\ifextended
\clearpage
\appendix

\section{Auxiliary Lemmas}

The first additional property that we need is that the initial value of $\kwf{target}$ only impact the final value of $\kwf{target}$ but no other variable.
\begin{lemma}
\label{lem:no-target}
For all Stan statement $\mathit{stmt}$,
for all variable $x$ different from $\kwf{target}$,
then for all values $t_1$ and $t_2$,
$$
\sem{\mathit{stmt}}_{\gamma[\kwf{target} \leftarrow t_1]}(x)
=
\sem{\mathit{stmt}}_{\gamma[\kwf{target} \leftarrow t_2]}(x)
$$
\end{lemma}
\begin{proof}
  The proof is done by induction on the structure of the statements and number of reductions.
  We detail just the most interesting cases.

  \paragraph{Assignment}
  Evaluating $\sassign{x}{e}$ does not update $\kwf{target}$ and according to \Cref{hyp:target}, the expression~$e$ does not depends on the value of target.
  So ${\sem{e}_{\gamma[\kwf{target} \leftarrow t_1]} = \sem{e}_{\gamma[\kwf{target} \leftarrow t_2]}}$ and thus
  $
  \sem{\sassign{x}{e}}_{\gamma[\kwf{target} \leftarrow t_1]}(x)
  =
  \sem{\sassign{x}{e}}_{\gamma[\kwf{target} \leftarrow t_2]}(x)
  $.

  \paragraph{Target Update}
  Since the evaluation of $\starget{e}$ only updates the value of $\kwf{target}$, we have by definition 
  
  \begin{small}
  $$
  \sem{\starget{e}}_{\gamma[\kwf{target} \leftarrow t_1]}(x) = \sem{\starget{e}}_{\gamma[\kwf{target} \leftarrow t_2]}(x)
  $$
  \end{small}

  \paragraph{Sequence}
  For a sequence $\sseq{\mathit{stmt}_1}{\mathit{stmt}_2}$ we have:

  \begin{small}
    $$
    \begin{array}{@{}l}
      \sem{\sseq{\mathit{stmt}_1}{\mathit{stmt}_2}}_{\gamma[\kwf{target} \leftarrow t_1]}(x)
      \\\qquad
      =
      \sem{\mathit{stmt}_2}_{\sem{\mathit{stmt}_1}_{\gamma[\kwf{target} \leftarrow t_1]}}(x)
      \quad\{ \text{ by definition } \}
      \\\qquad
      = \sem{\mathit{stmt}_2}_{\sem{\mathit{stmt}_1}_{\gamma[\kwf{target} \leftarrow t_2]}}(x)
      \quad\{ \text{ by induction on $\mathit{stmt_1}$ } \}
      \\\qquad
      = \sem{\sseq{\mathit{stmt}_1}{\mathit{stmt}_2}}_{\gamma[\kwf{target} \leftarrow t_2]}(x)
      \quad\{ \text{ by definition } \}
    \end{array}
    $$
  \end{small}

  \paragraph{Loop}
  For a loop $\swhile{e}{s}$, by induction on the structure of the statements, for all environment $\gamma$, values $t_1$ and $t_2$ and variable $x\neq \kwf{target}$, we have $\sem{s}_{\gamma[\kwf{target} \leftarrow t_1]}(x) = \sem{s}_{\gamma[\kwf{target} \leftarrow t_2]}(x)$, which implies:

  \begin{small}
  $$
  \sem{s}_{\gamma[\kwf{target} \leftarrow t_1]} = \sem{s}_{\gamma}[\kwf{target} \leftarrow \sem{s}_{\gamma[\kwf{target} \leftarrow t_1]}]
  $$
  \end{small}

  In addition, by \Cref{hyp:target} we know that for all environment $\gamma$,
  $\sem{e}_{\gamma[\kwf{target} \leftarrow t_1]} = \sem{e}_{\gamma[\kwf{target} \leftarrow t_2]}$.

  By \Cref{hyp:term}, there is a finite number~$n$ of loop iterations.
  So we can prove by induction on the number of remaining iterations that
  for all $t_1$, $t_2$:
  \begin{small}
  $$
  \sem{\swhile{e}{s}}_{\gamma^n[\kwf{target} \leftarrow t_1]}(x)
  =
  \sem{\swhile{e}{s}}_{\gamma^n[\kwf{target} \leftarrow t_2]}(x)
  $$
  \end{small}
  where for all $0 \leq i < n$, $\gamma^i = \sem{s}_{\gamma^{i+1}}$, and $\gamma^n = \gamma$.

  If $n = 0$, $\sem{e}_{\gamma^0[\kwf{target} \leftarrow t_1]} = \sem{e}_{\gamma^0[\kwf{target} \leftarrow t_2]} = 0$.
  So by definition of the semantics we have:
  \begin{small}
  $$
  \gamma^0[\kwf{target} \leftarrow t_1](x) = \gamma^0[\kwf{target} \leftarrow t_2](x)
  $$
  \end{small}

  For the inductive case, if $n = i + 1$, there are some iterations left so $\sem{e}_{\gamma^{i+1}[\kwf{target} \leftarrow t_1]} = \sem{e}_{\gamma^{i+1}[\kwf{target} \leftarrow t_2]} \not= 0$.
  So by definition of the semantics we have:

  \begin{small}
    $$
    \begin{array}{@{}l}
    \sem{\swhile{e}{s}}_{\gamma^{i+1}[\kwf{target} \leftarrow t_1]}(x)
    \\\quad
    = \sem{\swhile{e}{s}}_{\sem{s}_{\gamma^{i+1}[\kwf{target} \leftarrow t_1]}}(x)
      \quad \{ \text{ by definition } \}
    \\\quad
    = \{ \text{ by definition of $\gamma^i$ and induction on $s$ } \}
      \\\quad\phantom{=!}
      \sem{\swhile{e}{s}}_{\gamma^{i}[\kwf{target} \leftarrow \sem{s}_{\gamma^{i+1}[\kwf{target} \leftarrow t_1]}(\kwf{target})]}(x)
    \\\quad
    = \sem{\swhile{e}{s}}_{\gamma^{i}[\kwf{target} \leftarrow t_2]}(x)
       \quad \{ \text{ induction on $s$ } \}
    \\\quad
    = \{ \text{ by definition of $\gamma^i$ and induction on $s$ } \}
      \\\quad\phantom{=!}
      \sem{\swhile{e}{s}}_{\sem{s}_{\gamma^{i+1}}[\kwf{target} \leftarrow t_2]}(x)
    \\\quad
    = \sem{\swhile{e}{s}}_{\gamma^{i+1}[\kwf{target} \leftarrow t_2]}(x)
      \quad \{ \text{ by definition } \}
    \end{array}
    $$
  \end{small}
\end{proof}

The next lemma state that every statement can be evaluated in an environment where $\kwf{target}$ is set to zero.

\begin{lemma}
\label{lem:stan-target}
For all Stan statement $\mathit{stmt}$  and any real value~$t$ we have:
$$
    \sem{\mathit{stmt}}_{\gamma[\kwf{target} \leftarrow t]}(\kwf{target}) = t +  \sem{\mathit{stmt}}_{\gamma[\kwf{target} \leftarrow 0]}(\kwf{target})
    $$
\end{lemma}

\begin{proof}
  The proof is done by induction on the structure of the statements and number of reductions.

  \paragraph{Assignment.}
  Evaluating $\sassign{x}{e}$ does not update $\kwf{target}$.
  Therefore, for all $t$,
  ${\sem{\sassign{x}{e}}_{\gamma[\kwf{target} \leftarrow t]}(\kwf{target}) = t}$ and thus
  $$
  \sem{\sassign{x}{e}}_{\gamma[\kwf{target} \leftarrow t]}(\kwf{target}) = t +  \sem{\sassign{x}{e}}_{\gamma[\kwf{target} \leftarrow 0]}(\kwf{target})
  $$

  \paragraph{Target Update}
  Since the evaluation of $\starget{e}$ only updates the value of $\kwf{target}$, we have by definition

  \begin{small}
  $$
  \begin{array}{@{}l}
    \sem{\starget{e}}_{\gamma[\kwf{target} \leftarrow t]}(\kwf{target}) \\
    \qquad = t + \sem{\starget{e}}_{\gamma[\kwf{target} \leftarrow 0]}(\kwf{target})
  \end{array}
  $$
\end{small}

  \paragraph{Sequence}
  Let $\gamma_1 = \sem{\mathit{stmt}_1}_{\gamma[\kwf{target} \leftarrow 0]}$.
  By induction, we have $\sem{\mathit{stmt}_1}_{\gamma[\kwf{target} \leftarrow t]} = t + \gamma_1(\kwf{target})$.

  \begin{small}
  $$
  \begin{array}{@{}l}
    \sem{\sseq{\mathit{stmt}_1}{\mathit{stmt}_2}}_{\gamma[\kwf{target} \leftarrow t]}(\kwf{target})
    \\\quad
    = \sem{\mathit{stmt}_2}_{\sem{\mathit{stmt}_1}_{\gamma[\kwf{target} \leftarrow t]}}(\kwf{target})
      \quad \{\text{ by definition } \}
    \\\quad
    = \{ \text{ by induction on $\mathit{stmt}_2$ } \}
      \\\quad\phantom{=!}
      {\def\arraystretch{1.4}\begin{array}[t]{@{}l}
      \sem{\mathit{stmt}_1}_{\gamma[\kwf{target} \leftarrow t]}(\kwf{target}) \\
      + \sem{\mathit{stmt}_2}_{\sem{\mathit{stmt}_1}_{\gamma[\kwf{target} \leftarrow t]}[\kwf{target} \leftarrow 0]}(\kwf{target})
      \end{array}}
    \\\quad
    = \{ \text{ by induction on $\mathit{stmt}_1$ } \}
      \\\quad\phantom{=!}
      {\def\arraystretch{1.4}\begin{array}[t]{@{}l}
      t +
      \sem{\mathit{stmt_1}}_{\gamma[\kwf{target} \leftarrow 0]}(\kwf{target}) \\
      + \sem{\mathit{stmt}_2}_{\sem{\mathit{stmt}_1}_{\gamma[\kwf{target} \leftarrow t]}[\kwf{target} \leftarrow 0]}(\kwf{target})
      \end{array}}
    \\\quad
= \{ \text{ by \Cref{lem:no-target} and induction on $\mathit{stmt}_2$ } \}
      \\\quad\phantom{=!}
      t +
      \sem{\mathit{stmt}_2}_{\sem{\mathit{stmt}_1}_{\gamma[\kwf{target} \leftarrow 0]}[\kwf{target} \leftarrow 0]}(\kwf{target})
    \\\quad
    = t + \sem{\sseq{\mathit{stmt}_1}{\mathit{stmt}_2}}_{\gamma[\kwf{target} \leftarrow 0]}(\kwf{target})
      \quad \{\text{ by definition } \}
  \end{array}
  $$
\end{small}

  \paragraph{Loops}
  By \Cref{hyp:term}, there is a finite number~$n$ of loop iterations.
  So we can prove by induction on the number of remaining iterations that
  \begin{small}
  $$
  \begin{array}{@{}l}
  \sem{\swhile{e}{s}}_{\gamma^n[\kwf{target} \leftarrow t]}(\kwf{target})
  \\\qquad =
  t + \sem{\swhile{e}{s}}_{\gamma^n[\kwf{target} \leftarrow 0]}(\kwf{target})
  \end{array}
  $$
  \end{small}
  where for all $0 \leq i < n$, $\gamma^i = \sem{s}_{\gamma^{i+1}}$, and $\gamma^n = \gamma$.

  If $n = 0$, $\sem{e}_{\gamma^0[\kwf{target} \leftarrow t]} = \sem{e}_{\gamma^0[\kwf{target} \leftarrow 0]} = 0$.
  So by definition of the semantics we have:
  \begin{small}
  $$
  \gamma^0[\kwf{target} \leftarrow t](\kwf{target}) = t + \gamma^0[\kwf{target} \leftarrow 0](\kwf{target})
  $$
  \end{small}

  For the inductive case, if $n = i + 1$, there are some iterations left so by definition of the semantics we have:

  \begin{small}
  $$
  \begin{array}{@{}l}
  \sem{\swhile{e}{s}}_{\gamma^{i+1}[\kwf{target} \leftarrow t]}(\kwf{target})
  \\\quad
  = \sem{\swhile{e}{s}}_{\sem{s}_{\gamma^{i+1}[\kwf{target} \leftarrow t]}}(\kwf{target})
    \quad \{\text{ by definition } \}
  \\\quad
  = \{ \text{ by definition of $\gamma^i$ and \Cref{lem:no-target} } \}
    \\\quad\phantom{=!}
    \sem{\swhile{e}{s}}_{\gamma^i[\kwf{target} \leftarrow \sem{s}_{\gamma^{i+1}[\kwf{target \leftarrow t}]}(\kwf{target})]}(\kwf{target})
  \\\quad
  = \{ \text{ by induction on the number of reductions } \}
    \\\quad\phantom{=!}
    {\def\arraystretch{1.4}\begin{array}[t]{@{}l}
  \sem{s}_{\gamma^{i+1}[\kwf{target \leftarrow t}]}(\kwf{target}) \\
  + \sem{\swhile{e}{s}}_{\gamma^i[\kwf{target} \leftarrow 0]}(\kwf{target})
  \end{array}}
  \\\quad
  = \{ \text{ by induction on $s$ } \}
    \\\quad\phantom{=!}
    {\def\arraystretch{1.4}\begin{array}[t]{@{}l}
    t + \sem{s}_{\gamma^{i+1}[\kwf{target} \leftarrow 0]}(\kwf{target}) \\
    + \sem{\swhile{e}{s}}_{\gamma^i[\kwf{target} \leftarrow 0]}(\kwf{target})
    \end{array}}
  \\\quad
  = \{ \text{ by definition of $\gamma^i$ } \}
    \\\quad\phantom{=!}
    {\def\arraystretch{1.4}\begin{array}[t]{@{}l}
    t + \sem{s}_{\gamma^{i+1}[\kwf{target} \leftarrow 0]}(\kwf{target}) \\
    + \sem{\swhile{e}{s}}_{\sem{s}_{\gamma^{i+1}}[\kwf{target} \leftarrow 0]}(\kwf{target})
    \end{array}}
  \\\quad
  = \{ \text{ by \Cref{lem:no-target} } \}
    \\\quad\phantom{=!}
    {\def\arraystretch{1.4}\begin{array}[t]{@{}l}
    t + \sem{s}_{\gamma^{i+1}[\kwf{target} \leftarrow 0]}(\kwf{target}) \\
    + \sem{\swhile{e}{s}}_{\sem{s}_{\gamma^{i+1}[\kwf{target} \leftarrow 0]}[\kwf{target} \leftarrow 0]}(\kwf{target})
    \end{array}}
  \\\quad
  = \{ \text{ by induction } \}
    \\\quad\phantom{=!}
    {\def\arraystretch{1.4}\begin{array}[t]{@{}l}
    t + \sem{\swhile{e}{s}}_{\gamma^{i+1}[\kwf{target} \leftarrow 0]}(\kwf{target})
  \end{array}}
  \end{array}
  $$
\end{small}

\end{proof}

\section{Proofs: Correctness of the Compilation}
\label{apx:correctness}

\begin{lemma}
  \label{lem:params-apx}
  For all Stan programs $p$ with $\mathit{stmt} = \secmodel{p}$ and $\mathcal{P} = \secparams{p}$,  and environments~$\gamma$:
\begin{small}
  $$
  \psem{\pcomp{p}}_\gamma \propto
    \lambda U. \int_U \psem{\scomp{\ereturn{\eunit}}{\mathit{stmt}}}_{\gamma[\mathcal{P} \leftarrow \theta]}(\{\eunit\})d\theta
  $$
\end{small}
\end{lemma}
\begin{proof}
  Let $\mathcal{P} = x_1, ..., x_n$.
  By definition of the compilation function, $\pcomp{p}$ has the shape:
\begin{small}
  $$
  \elet{x_1}{D_1}{
\dots
      \elet{x_n}{D_n}{
        \scomp{\ereturn{\mathcal{P}}}{\mathit{stmt}}}}
$$
  \end{small}
  where for each parameter $x_i$, the distribution $D_i$ is either \python{uniform} or \python{improper_uniform}.
  In both cases the corresponding density is constant w.r.t. the Lebesgue measure on its domain.

  Since the kernels defined by GProb expressions are always s-finite~\cite{staton_2017}, from the semantics of GProb (\Cref{sec:formal_generative_language}) and the Fubini-Tonelli theorem, we have:

  \begin{small}
  $$
  \def\arraystretch{1.8}
  \begin{array}{@{}l}
    \psem{\pcomp{p}}_\gamma 
    \\ \;= \psem{\elet{x_1}{D_1}{
\dots
          \elet{x_n}{D_n}{
            \scomp{\ereturn{\mathcal{P}}}{\mathit{stmt}}}}}_\gamma
\\ \;=
       \lambda U.\!\!
       \displaystyle\int_{X_1}
       {\def\arraystretch{1.4}\begin{array}[t]{@{}l}
       \!\!\!\! D_1(dv_1) \,... \\
       \int_{X_n}
       {\def\arraystretch{1.4}\begin{array}[t]{@{}l}
       \!\!\! D_n(dv_n) \\
       \psem{\scomp{\ereturn{\mathcal{P}}}{\mathit{stmt}}}_{\gamma[(x_1,...,x_n) \leftarrow (v_1, ..., v_n)]}(U)
      \end{array}}
       \end{array}}
\\ \;\propto
       \lambda U.\!\!
       \displaystyle\int_{X_1} \!\!\!\!...\!\!\!
       \int_{X_n}\!\!
       \psem{\scomp{\ereturn{\mathcal{P}}}{\mathit{stmt}}}_{\gamma[(x_1,...,x_n) \leftarrow (v_1, ..., v_n)]}(U)\ dv_1 ... dv_n
    \\ \;=
       \lambda U.\!\!
       \displaystyle\int_{X = X_1 \times ... \times X_n}
       \psem{\scomp{\ereturn{\mathcal{P}}}{\mathit{stmt}}}_{{\gamma[(x_1,...,x_n) \leftarrow \theta]}}(U)\ d\theta
  \end{array}
  $$
\end{small}

  The evaluation of $\psem{\scomp{\ereturn{\mathcal{P}}}{\mathit{stmt}}}_{\gamma[\mathcal{P} \leftarrow \theta]}(U)$ terminates with the value $\psem{\ereturn{\mathcal{P}}}_{\gamma'[\mathcal{P} \leftarrow \theta]}(U)$ where $\gamma'$ is the environment obtained after the evaluation of the model statements $\mathit{stmt}$. 

  However, in Stan, parameters declared in $\mathcal{P}$ cannot appear in the left-hand side of an assignment. The evaluation of the model statements $\mathit{stmt}$ cannot update the value of the parameters.
  We can thus simplify $\psem{\ereturn{\mathcal{P}}}_{\gamma'[\mathcal{P} \leftarrow \theta]}(U)$ into $\psem{\ereturn{\mathcal{P}}}_{[\mathcal{P} \leftarrow \theta]}(U)$.
  The evaluation of the compiled code can then be decomposed as follows:
  \begin{small}
  $$
  \def\arraystretch{1.8}
  \begin{array}{l}
    \psem{\scomp{\ereturn{\mathcal{P}}}{\mathit{stmt}}}_{\gamma[\mathcal{P} \leftarrow \theta]}(U)
    \\ \qquad =
    \displaystyle\int_X
      {\def\arraystretch{1.4}\begin{array}[t]{@{}l}
      \psem{\scomp{\ereturn{\eunit}}{\mathit{stmt}}}_{\gamma[\mathcal{P} \leftarrow \theta]}(dx)\ \times \\
      \psem{\ereturn{\mathcal{P}}}_{[\mathcal{P} \leftarrow \theta]}(U)
      \end{array}}
    \\ \qquad =
    \psem{\scomp{\ereturn{\eunit}}{\mathit{stmt}}}_{\gamma[\mathcal{P} \leftarrow \theta]}(\{\eunit\}) \times
    \delta_{\theta}(U)
  \end{array}
  $$
\end{small}

  \noindent
  Going back to the previous equation, we now have
  \begin{small}
  $$
  \def\arraystretch{1.8}
  \begin{array}{@{}l@{}}
    \psem{\pcomp{p}}_\gamma
    \\\quad
    =
\lambda U. \displaystyle\int_{X}
    \psem{\scomp{\ereturn{\eunit}}{\mathit{stmt}}}_{\gamma[\mathcal{P} \leftarrow \theta]}(\{\eunit\}) \times
    \delta_{\theta}(U)\ d\theta
    \\\quad
    =
    \lambda U. \displaystyle\int_U \psem{\scomp{\ereturn{\eunit}}{\mathit{stmt}}}_{\gamma[\mathcal{P} \leftarrow \theta]}(\{\eunit\})\ d\theta
  \end{array}
  $$
\end{small}
\end{proof}

\begin{lemma}
  \label{lem:body-apx}
  For all Stan statements $\mathit{stmt}$ compiled with a continuation $k$,
  if $\gamma(\kwf{target}) = 0$, and $\sem{\mathit{stmt}}_\gamma = \gamma'$,
$$
  \psem{\scomp{k}{\mathit{stmt}}}_{\gamma} = \lambda U.
  \exp(\gamma'(\kwf{target})) \times \psem{k}_{\gamma'[\kwf{target} \leftarrow 0]}(U)
  $$
\end{lemma}
\begin{proof}
  The proof is done by induction on the structure of $\mathit{stmt}$ and number of reductions using the definition of the compilation function (\Cref{sec:compil-all}) and the semantics of GProb.

  \paragraph{Assignment.}
  Evaluating $\sassign{x}{e}$ does not update $\kwf{target}$ and its initial value is $0$ by hypothesis. With $\gamma' = \sem{\sassign{x}{e}}_\gamma$ we have $\gamma'[\kwf{target} \leftarrow 0] = \gamma'$ and $\exp(\gamma'(\kwf{target})) = 1$. Then from GProb's semantics we have:

    \begin{small}
    \begin{center}$
    \def\arraystretch{1.2}
    \begin{array}{@{}l@{\;}c@{\;\;}l}
      \psem{\scomp{k}{\sassign{x}{e}}}_{\gamma}
      & = & \{ \text{ by definition of $\scomp{k}{.}$ } \}
      \\ &   & \psem{\elet{x}{\ereturn{e}}{k}}_{\gamma}
      \\ & = & \{ \text{ by definition of the semantics } \}
      \\ &   & \lambda U.
      \displaystyle\int_X \psem{\ereturn{e}}_{\gamma}(dv) \times \psem{k}_{\gamma[x \leftarrow v]} (U)
      \\ & = & \{ \text{ by definition of the semantics } \}
      \\ &   & \lambda U.
      \displaystyle\int_X \delta_{\sem{e}_{\gamma}}(dv) \times \psem{k}_{\gamma[x \leftarrow v]} (U)
      \\ & = & \{ \text{ by the integration of the $\delta$ distribution } \}
      \\ &   & \lambda U.
      1 \times \psem{k}_{\gamma[x \leftarrow \sem{e}_{\gamma}]} (U)
      \\ & = & \{ \text{ by the semantics of $\sem{\sassign{x}{e}}_{\gamma}$ } \}
      \\ &   & \lambda U.
      1 \times \psem{k}_{\sem{\sassign{x}{e}}_{\gamma}} (U)
      \\ & = & \{ \text{ by definition of $\gamma'$ } \}
      \\ &   & \lambda U.
      \exp(\gamma'(\kwf{target})) \times \psem{k}_{\gamma'[\kwf{target} \leftarrow 0]} (U)
    \end{array}
    $\end{center}
    \end{small}

  \paragraph{Target update.}
  Since the evaluation of $\starget{e}$ only updates the value of $\kwf{target}$ and its initial value is $0$, with $\gamma' = \sem{\starget{e}}_{\gamma}$ we have $\gamma = \gamma'[\kwf{target} \leftarrow 0]$, and $\gamma'(\kwf{target}) = \sem{e}_\gamma$. Then from GProb's semantics we have:
    \begin{small}
    $$
    \def\arraystretch{1.2}
    \begin{array}{@{}l@{}}
    \psem{\scomp{k}{\starget{e}}}_{\gamma}\\
    \begin{array}{@{}l@{\;}c@{\;\;}l}
      \quad & = & \{ \text{ by definition of $\scomp{k}{.}$ } \}
      \\ &  & \psem{\elet{\eunit}{\efactor{e}}{k}}_{\gamma}
      \\ & = & \{ \text{ by definition of the semantics } \}
      \\ &   & \lambda U.
      \displaystyle\int_\eunit \exp(\sem{e}_{\gamma}) \delta_{\eunit}(dv) \times \psem{k}_{\gamma} (U)
      \\ & = &\{ \text{ by the integration of the $\delta$ distribution } \}
      \\ &   &  \lambda U.
      \exp(\sem{e}_{\gamma}) \times \psem{k}_{\gamma} (U)
      \\ & = & \{ \text{ by the semantics of $\sem{\starget{e}}_{\gamma}$ } \}
      \\ &   & \lambda U.
      \exp(\sem{\starget{e}}_{\gamma}(\kwf{target})) \times \psem{k}_{\gamma} (U)
      \\ & = & \{ \text{ by definition of $\gamma'$ } \}
      \\ &   & \lambda U.
      \exp(\gamma'(\kwf{target})) \times \psem{k}_{\gamma'[\kwf{target} \leftarrow 0]} (U)
   \end{array}
  \end{array}
    $$
  \end{small}

  \paragraph{Sequence.}
    If $\gamma_1 = \sem{\mathit{stmt}_1}_\gamma$ and $\gamma_2 = \sem{\mathit{stmt}_2}_{\gamma_1[\kwf{target} \leftarrow 0]}$, the induction hypothesis and the semantics of GProb yield:

    \begin{small}
    $$
    \def\arraystretch{1.4}
    \begin{array}{l}
      \psem{\scomp{k}{\sseq{\mathit{stmt}_1}{\mathit{stmt}_2}}}_{\gamma}
      = \psem{\scomp{\scomp{k}{\mathit{stmt_2}}}{\mathit{stmt_1}}}_{\gamma}
      \\ \; = \{ \text{ by induction on $\mathit{stmt}_1$ and definition of $\gamma_1$ } \}
      \\ \; \phantom{=~}
      \lambda U.
      {\def\arraystretch{1.1}
       \begin{array}[t]{@{}l}
         \exp(\gamma_1(\kwf{target})) \times \psem{\scomp{k}{\mathit{stmt_2}}}_{\gamma_1[\kwf{target} \leftarrow 0]}(U) \\
       \end{array}}
      \\ \; = \{ \text{ by induction on $\mathit{stmt}_2$ and definition of $\gamma_2$ } \}
      \\ \; \phantom{=~}
      \lambda U.
      {\def\arraystretch{1.1}
       \begin{array}[t]{@{}l}
         \exp(\gamma_1(\kwf{target})) \times \exp(\gamma_2(\kwf{target})) \times \psem{k}_{\gamma_2[\kwf{target} \leftarrow 0]}(U) \\
       \end{array}}
      \\ \; = \lambda U.
      {\def\arraystretch{1.1}
       \begin{array}[t]{@{}l}
         \exp(\gamma_1(\kwf{target}) + \gamma_2(\kwf{target})) \times \psem{k}_{\gamma_2[\kwf{target} \leftarrow 0]}(U) \\
       \end{array}}
   \end{array}
    $$
  \end{small}

  On the other hand, from \Cref{lem:stan-target} we have:\\
  $\sem{\sseq{\mathit{stmt}_1}{\mathit{stmt}_2}}_\gamma(\kwf{target}) = \gamma_1(\kwf{target}) + \gamma_2(\kwf{target})$ which conclude the proof of this case.

  \paragraph{Loop.} For the case $\swhile{e}{s}$, we first look at the semantics of the loop on the Stan side.
  We can prove that

  \begin{small}
  \begin{equation}
  \label{eq:while-stan}
  \def\arraystretch{1.4}
  \begin{array}{l}
  \sem{\swhile{e}{s}}_{\gamma}(\kwf{target})
  = \displaystyle\sum_{i=0}^{n-1} \gamma^{i}(\kwf{target})
  \end{array}
  \end{equation}
  \end{small}
  \noindent
  where $n$ is the number of loop iterations, for all $0 \leq i < n$, $\gamma^i = \sem{s}_{\gamma^{i+1}[\kwf{target} \leftarrow 0]}$, and $\gamma^n = \gamma$.
  By \Cref{hyp:term} the programs terminate, so $n$ is finite.

  This is a direct consequence of the following property since $\gamma(\kwf{target}) = 0$:
  \begin{small}
  \begin{equation}
  \label{eq:while-stan-ind}
  \def\arraystretch{1.4}
  \begin{array}{l}
  \sem{\swhile{e}{s}}_{\gamma^n[\kwf{target} \leftarrow 0]}(\kwf{target})
  = \displaystyle\sum_{i=0}^{n-1} \gamma^{i}(\kwf{target})
  \end{array}
  \end{equation}
  \end{small}

  The proof of \cref{eq:while-stan-ind} is done by induction on the number of remaining iterations.
  If $n = 0$, the condition of the loop is false.
  Since the condition of the loop cannot depend on the value of \kwf{target}, it means that $\sem{e}_{\gamma^0} = \sem{e}_{\gamma^0[\kwf{target} \leftarrow 0]} = 0$, so by application of the semantics we have directly:
  \begin{small}
  $$
  \def\arraystretch{1.4}
  \begin{array}{l}
  \sem{\swhile{e}{s}}_{\gamma^0[\kwf{target} \leftarrow 0]}(\kwf{target})
= 0
  \end{array}
  $$
  \end{small}

  For the inductive case~($n = i + 1$), there is some loop iterations remaining and thus $\sem{e}_{\gamma^{i+1}} = \sem{e}_{\gamma^{i+1}[\kwf{target} \leftarrow 0]} \not= 0$. By application on the semantics we have:
  \begin{small}
  $$
  \def\arraystretch{1.4}
  \begin{array}{l}
  \sem{\swhile{e}{s}}_{\gamma^{i+1}[\kwf{target} \leftarrow 0]}(\kwf{target})
  \\\qquad
  = \sem{\swhile{e}{s}}_{\sem{s}_{\gamma^{i+1}[\kwf{target} \leftarrow 0]}}
  \\\qquad
  = \sem{\swhile{e}{s}}_{\gamma^i} \quad\{ \text{ by definition of $\gamma^i$ }\}
  \\\qquad
  = \gamma^i + \sem{\swhile{e}{s}}_{\gamma^i[\kwf{target} \leftarrow 0]}
    \quad\{ \text{ by \Cref{lem:stan-target} }\}
  \\\qquad
  = \{ \text{ by induction on the number of iterations }\}
  \\\qquad\phantom{=~}
    \gamma^i + \displaystyle\sum_{j=0}^{i-1} \gamma^{j}(\kwf{target})
= \displaystyle\sum_{j=0}^{i} \gamma^{j}(\kwf{target})
  \end{array}
  $$
  \end{small}

  Let's look at the loop $\ewhile{\mathcal{X}}{e}{\scomp{\ereturn{\mathcal{X}}}{s}}$ where $\mathcal{X}$ is the variables updated by~$s$ ($\mathit{lhs}(s) = \mathcal{X}$).
  We will use the two following properties.
  For all $\gamma$, if $\sem{e}_\gamma = 0$
\begin{small}
  \begin{equation}
    \label{eq:while0}
      \def\arraystretch{1.4}
      \begin{array}[c]{l}
        \psem{\ewhile{\mathcal{X}}{e}{\scomp{\ereturn{\mathcal{X}}}{s}}}_{\gamma}
        = \delta_{\gamma(\mathcal{X})}
      \end{array}
  \end{equation}
  \end{small}
  For all $\gamma$, if $\sem{e}_\gamma \not= 0$ and $\sem{s}_\gamma = \gamma'$:
  \begin{small}
  \begin{equation}
    \label{eq:while1}
      \def\arraystretch{1.4}
      \begin{array}[c]{l}
        \psem{\ewhile{\mathcal{X}}{e}{\scomp{\ereturn{\mathcal{X}}}{s}}}_{\gamma}
        \\\qquad
        =
        \lambda U. {\def\arraystretch{1.4}\begin{array}[t]{@{}l}
                     \exp(\gamma'(\kwf{target}))\ \times \\
                     \psem{\ewhile{\mathcal{X}}{e}{\scomp{\ereturn{\mathcal{X}}}{s}}}_{\gamma'[\kwf{target} \leftarrow 0]}(U)
                   \end{array}}
      \end{array}
  \end{equation}
  \end{small}
  The proofs of both of these properties start by unfolding the definition of the semantics:
  \begin{small}
  $$
  \def\arraystretch{1.4}
  \begin{array}{l}
    \psem{\ewhile{\mathcal{X}}{e}{\scomp{\ereturn{\mathcal{X}}}{s}}}_{\gamma}
    \\\qquad
    =
    \lambda U. {\def\arraystretch{1.4}\begin{array}[t]{@{}l}
                 \mathit{if}\ \sem{e}_\gamma = 0\ \mathit{then}\ \delta_{\gamma(\mathcal{X})}(U)\\
                 \mathit{else}\! \displaystyle\int_X
                 {\def\arraystretch{1.4}\begin{array}[t]{@{}l}
                   \psem{\scomp{\ereturn{\mathcal{X}}}{s}}_\gamma(dv)\ \times \\
                   \psem{\ewhile{\mathcal{X}}{e}{\scomp{\ereturn{\mathcal{X}}}{s}}}_{\gamma[\mathcal{X}\leftarrow v]}(U)
                 \end{array}}
                \end{array}}
  \end{array}
  $$
  \end{small}

  From there, the proof of \cref{eq:while0} where $\sem{e}_\gamma = 0$ is trivial.
  For \cref{eq:while1} where $\sem{e}_\gamma \not= 0$ and $\sem{s}_\gamma = \gamma'$,
  we apply the induction hypothesis on $\psem{\scomp{\ereturn{\mathcal{X}}}{s}}_\gamma$:
  \begin{small}
  $$
  \def\arraystretch{1.4}
  \begin{array}{l}
    \psem{\ewhile{\mathcal{X}}{e}{\scomp{\ereturn{\mathcal{X}}}{s}}}_{\gamma}
    \\
    = \{ \text{ by simplification of $\sem{e}_\gamma \not= 0$} \}
    \\\phantom{=~}
    \lambda U. \displaystyle\int_X
                 {\def\arraystretch{1.4}\begin{array}[t]{@{}l}
                   \psem{\scomp{\ereturn{\mathcal{X}}}{s}}_\gamma(dv)\ \times \\
                   \psem{\ewhile{\mathcal{X}}{e}{\scomp{\ereturn{\mathcal{X}}}{s}}}_{\gamma[\mathcal{X}\leftarrow v]}(U)
                 \end{array}}
    \\
    = \{ \text{ by induction on the structure of the statement} \}
    \\\phantom{=~}
    \lambda U. \displaystyle\int_X
                 {\def\arraystretch{1.4}\begin{array}[t]{@{}l}
                   (\exp(\gamma'(\kwf{target})) \times \psem{\ereturn{\mathcal{X}}}_{\gamma'[\kwf{target} \leftarrow 0]})(dv)\ \times \\
                   \psem{\ewhile{\mathcal{X}}{e}{\scomp{\ereturn{\mathcal{X}}}{s}}}_{\gamma[\mathcal{X}\leftarrow v]}(U)
                 \end{array}}
    \\
    =\{ \text{ by definition of the semantics of \kwf{return}} \}
    \\\phantom{=~}
    \lambda U. \displaystyle\int_X
                 {\def\arraystretch{1.4}\begin{array}[t]{@{}l}
                   (\exp(\gamma'(\kwf{target})) \times \delta_{\gamma'[\kwf{target} \leftarrow 0](\mathcal{X})})(dv)\ \times \\
                   \psem{\ewhile{\mathcal{X}}{e}{\scomp{\ereturn{\mathcal{X}}}{s}}}_{\gamma[\mathcal{X}\leftarrow v]}(U)
                 \end{array}}
    \\
    =\{ \text{ by the integration of the $\delta$ distribution } \}
    \\\phantom{=~}
    \lambda U. {\def\arraystretch{1.4}\begin{array}[t]{@{}l}
                 \exp(\gamma'(\kwf{target}))\ \times \\
                 \psem{\ewhile{\mathcal{X}}{e}{\scomp{\ereturn{\mathcal{X}}}{s}}}_{\gamma[\mathcal{X} \leftarrow \gamma'[\kwf{target} \leftarrow 0](\mathcal{X})]}(U)
               \end{array}}
    \\
    =\{ \text{ the only difference between $\gamma$ and $\gamma'$ are in $\mathcal{X}$ by def of $\mathit{lhs}$ } \}
    \\\phantom{=~}
    \lambda U. {\def\arraystretch{1.4}\begin{array}[t]{@{}l}
                 \exp(\gamma'(\kwf{target}))\ \times \\
                 \psem{\ewhile{\mathcal{X}}{e}{\scomp{\ereturn{\mathcal{X}}}{s}}}_{\gamma'[\kwf{target} \leftarrow 0]}(U)
               \end{array}}
  \end{array}
  $$
  \end{small}

  We can now prove by induction on the number of loop iterations~$n$~(that we assume finite) that

  \begin{small}
  \begin{equation}
    \label{eq:while-gprob}
    \def\arraystretch{1.4}
    \begin{array}{l}
      \psem{\ewhile{\mathcal{X}}{e}{\scomp{\ereturn{\mathcal{X}}}{s}}}_{\gamma^n[\kwf{target} \leftarrow 0]}
      \\\qquad
      =
      \lambda U. \left(\displaystyle\prod_{i=0}^{n-1} \exp(\gamma^{i}(\kwf{target}))\right) \times \delta_{\gamma^0(\mathcal{X})}(U)
    \end{array}
  \end{equation}
  \end{small}

  \noindent
  where for all $0 \leq i < n$, $\gamma^i = \sem{s}_{\gamma^{i+1}[\kwf{target} \leftarrow 0]}$.

  In the base case of the induction~($n =0 $), there is no more iteration~($\sem{e}_{\gamma^0} = \sem{e}_{\gamma^0[\kwf{target} \leftarrow 0]} = 0$).
  By application of~\cref{eq:while0}, we have:
  \begin{small}
  $$
    \def\arraystretch{1.4}
    \begin{array}{l}
      \psem{\ewhile{\mathcal{X}}{e}{\scomp{\ereturn{\mathcal{X}}}{s}}}_{\gamma^0[\kwf{target} \leftarrow 0]}
=
      \delta_{\gamma(\mathcal{X})}
    \end{array}
  $$
  \end{small}

  For the inductive case~($n = i + 1$), we know that $\sem{e}_{\gamma^{i+1}} = \sem{e}_{\gamma^{i+1}[\kwf{target} \leftarrow 0]} \not= 0$ since there are some more loop iterations to execute. Therefore, we can apply \cref{eq:while1}:
  \begin{small}
    $$
    \def\arraystretch{1.4}
    \begin{array}{l}
      \psem{\ewhile{\mathcal{X}}{e}{\scomp{\ereturn{\mathcal{X}}}{s}}}_{\gamma^{i+1}[\kwf{target} \leftarrow 0]}
      \\\qquad
      = \{ \text{ by \cref{eq:while1} and definition of $\gamma^i$ } \}
      \\\qquad\phantom{=~}
      \lambda U. {\def\arraystretch{1.4}\begin{array}[t]{@{}l}
                   \exp(\gamma^i(\kwf{target}))\ \times \\
                   \psem{\ewhile{\mathcal{X}}{e}{\scomp{\ereturn{\mathcal{X}}}{s}}}_{\gamma^i[\kwf{target} \leftarrow 0]}(U)
                 \end{array}}
      \\\qquad
      = \{ \text{ by induction on the number of iterations }\}
      \\\qquad\phantom{=~}
      \lambda U. {\def\arraystretch{1.4}\begin{array}[t]{@{}l}
                   \exp(\gamma^i(\kwf{target})) \times \left(\displaystyle\prod_{j=0}^{i-1} \exp(\gamma^{j}(\kwf{target}))\right) \times \delta_{\gamma^0(\mathcal{X})}(U)
      \end{array}}
      \\\qquad
      = \lambda U. {\def\arraystretch{1.4}\begin{array}[t]{@{}l}
                    \left(\displaystyle\prod_{j=0}^{i} \exp(\gamma^{j}(\kwf{target}))\right) \times \delta_{\gamma^0(\mathcal{X})}(U)
      \end{array}}
    \end{array}
    $$
  \end{small}

  We can finally put all the pieces together to prove the case of $\swhile{e}{s}$.
We still assume that the loop has $n$ iterations.
  We first unroll the definition of compilation function and the semantics.
  \begin{small}
  $$
  \def\arraystretch{1.4}
  \begin{array}{l}
    \psem{\scomp{k}{\swhile{e}{s}}}_{\gamma}
    \\
    =
    \psem{\elet{\mathcal{X}}{\ewhile{\mathcal{X}}{e}{\scomp{\ereturn{\mathcal{X}}}{s}}}{k}}_{\gamma}
    \\
    =
    \lambda U.
    \displaystyle\int_{X}
       \psem{\ewhile{\mathcal{X}}{e}{\scomp{\ereturn{\mathcal{X}}}{s}}}_{\gamma}(dv)
       \times
       \psem{k}_{\gamma[\mathcal{X} \leftarrow v]}(U)

    \\
    = \{ \text{ by \cref{eq:while-gprob} with $\gamma^n = \gamma$ and because $\gamma(\kwf{target}) = 0$ } \}
    \\\phantom{=~}
    \lambda U.
    \displaystyle\int_{X}
                  \left(\displaystyle\prod_{i=0}^{n-1} \exp(\gamma^{i}(\kwf{target}))\right) \times \delta_{\gamma^0(\mathcal{X})}
       (dv)
       \times
       \psem{k}_{\gamma[\mathcal{X} \leftarrow v]}(U)
    \\
    = \{ \text{ by rewriting the product of exponential } \}
    \\\phantom{=~}
    \lambda U.
    \displaystyle\int_{X}
                  \exp(\displaystyle\sum_{i=0}^{n-1} \gamma^{i}(\kwf{target})) \times \delta_{\gamma^0(\mathcal{X})}
       (dv)
       \times
       \psem{k}_{\gamma[\mathcal{X} \leftarrow v]}(U)
    \\
    = \{ \text{ by \cref{eq:while-stan} with $\gamma' = \sem{\swhile{e}{s}}_{\gamma}$ } \}
    \\\phantom{=~}
    \lambda U.
    \displaystyle\int_{X}
                  \exp(\gamma'(\kwf{target})) \times \delta_{\gamma^0(\mathcal{X})}
       (dv)
       \times
       \psem{k}_{\gamma[\mathcal{X} \leftarrow v]}(U)
    \\
    = \{ \text{ by the integration of the $\delta$ distribution } \}
    \\\phantom{=~}
    \lambda U.
    \exp(\gamma'(\kwf{target}))
       \times
       \psem{k}_{\gamma[\mathcal{X} \leftarrow {\gamma^0(\mathcal{X})}]}(U)
    \\
    = \{ \text{ by definition of $\mathit{lhs}$ and since $\kwf{target} \not\in \mathcal{X}$ } \}
    \\\phantom{=~}
    \lambda U.
    \exp(\gamma'(\kwf{target}))
       \times
       \psem{k}_{\gamma^0[\kwf{target} \leftarrow 0]}(U)
    \\
    = \lambda U.
    \exp(\gamma'(\kwf{target}))
       \times
       \psem{k}_{\gamma'[\kwf{target} \leftarrow 0]}(U)

  \end{array}
  $$
  \end{small}

\end{proof}

\section{Evaluation}
\label{apx:evaluation}

We present here the results of the experiments for RQ2 and RQ3 described in \Cref{sec:evaluation} on the entire set of PosteriorDB models with reference, that is, without filtering out the models for which Stan does not pass the accuracy test.

\paragraph{RQ2: Accuracy}
\Cref{tab:eval_accuracy_full} presents the accuracy results.
For each run we report the mean relative error in green if the model pass the accuracy test, orange otherwise (the accuracy test fails if the relative error for one of the parameter is above~$0.3$).

Compared to the results presented in \Cref{sec:evaluation} we have two additional errors for the NumPyro backend: gp\_pois\_regr (missing \python{cov_exp_quad}) and diamonds\_diamonds (missing \python{student_t_lccdf}).
We also have two additional mismatches low\_dim\_gauss\_mix and hmm\_drive\_1 both due to an \python{ordered} constraint that is not supported in the versions of Pyro and NumPyro that we used.

For the Pyro backend two examples for which Stan returns a mismatch timed-out after 10h (indicated with a ``---'' in \Cref{tab:eval_accuracy_full}): hmm\_drive\_1
 and low\_dim\_gauss\_mix.

\paragraph{RQ3: Speed}
\Cref{tab:eval_speed_full} presents the speed results.
For each run we report the average duration in second and the standard deviation.

\newcommand{\green}[1]{{\color{green!60!black}#1}}
\newcommand{\orange}[1]{{\color{orange!90!black}#1}}
\newcommand{\gray}[1]{{\color{black!40!white}#1}}

\begin{table*}
  \caption{Comparing inference results with PosteriorDB references. For each column we report the mean relative error in green if the example pass the accuracy test, orange otherwise.}
  \label{tab:eval_accuracy_full}
\begin{small}
  \begin{tabular}{@{}llrrrrr@{}}

&&&\multicolumn{1}{c}{\textsc{Pyro}}&\multicolumn{3}{c}{\textsc{NumPyro}}
\\
\cmidrule(lr){4-4}
\cmidrule(lr){5-7}
\textsc{Model} & \textsc{Dataset} & \textsc{Stan} & \textsc{Compr.} & \textsc{Compr.} & \textsc{Mixed} & \textsc{Gener.}\\
    \toprule
    accel\_gp &            mcycle\_gp &     \green{ 0.01} &                      \emark &                         \emark &                 \emark &                      \emark \\
    arK &                  arK &     \green{ 0.03} &                   \green{ 0.01} &                      \green{ 0.01} &              \green{ 0.01} &                   \green{ 0.01} \\
 arma11 &                 arma &     \green{ 0.01} &                   \green{ 0.01} &                      \green{ 0.00} &              \green{ 0.01} &                   \green{ 0.01} \\
    blr &                sblrc &     \orange{ 0.98} &                   \green{ 0.01} &                      \green{ 0.01} &              \green{ 0.01} &                      \emark \\
    blr &                sblri &     \orange{ 0.63} &                   \green{ 0.01} &                      \green{ 0.01} &              \green{ 0.01} &                      \emark \\
diamonds &             diamonds &     \orange{ 1.10} &                      \emark &                         \emark &                 \emark &                      \emark \\
   dogs &                 dogs &     \green{ 0.10} &                   \green{ 0.01} &                      \green{ 0.00} &              \green{ 0.01} &                   \green{ 0.01} \\
dogs\_log &                 dogs &     \green{ 0.03} &                   \green{ 0.01} &                      \green{ 0.01} &              \green{ 0.03} &                      \emark \\
earn\_height &             earnings &     \green{ 0.01} &                   \green{ 0.02} &                      \green{ 0.00} &              \green{ 0.00} &                      \emark \\
eight\_schools\_centered &        eight\_schools &     \green{ 0.02} &                   \green{ 0.03} &                      \green{ 0.03} &              \green{ 0.02} &                   \green{ 0.02} \\
eight\_schools\_noncentered &        eight\_schools &     \green{ 0.01} &                   \green{ 0.01} &                      \green{ 0.01} &              \green{ 0.01} &                   \orange{ 0.12} \\
garch11 &                garch &     \green{ 0.01} &                   \orange{ 0.99} &                      \orange{ 0.22} &              \orange{ 0.22} &                      \emark \\
gp\_pois\_regr &         gp\_pois\_regr &     \orange{ 0.14} &                      \emark &                         \emark &                 \emark &                      \emark \\
gp\_regr &         gp\_pois\_regr &     \green{ 0.00} &                      \emark &                         \emark &                 \emark &                      \emark \\
hmm\_drive\_0 &  bball\_drive\_event\_0 &     \green{ 0.10} &                   \green{ 0.01} &                      \green{ 0.01} &              \green{ 0.01} &                      \emark \\
hmm\_drive\_1 &  bball\_drive\_event\_1 &     \orange{ 1.90} &                      --- &                     \orange{ 27.70} &             \orange{ 27.70} &                      \emark \\
hmm\_example &          hmm\_example &     \green{ 0.09} &                   \green{ 0.00} &                      \green{ 0.01} &              \green{ 0.01} &                      \emark \\
kidscore\_interaction &                kidiq &     \green{ 0.03} &                   \green{ 0.01} &                      \green{ 0.01} &              \green{ 0.01} &                      \emark \\
kidscore\_interaction\_c &  kidiq\_with\_mom\_work &     \orange{ 0.09} &                   \green{ 0.00} &                      \green{ 0.01} &              \green{ 0.01} &                      \emark \\
kidscore\_interaction\_c2 &  kidiq\_with\_mom\_work &     \green{ 0.05} &                   \green{ 0.00} &                      \green{ 0.01} &              \green{ 0.01} &                      \emark \\
kidscore\_interaction\_z &  kidiq\_with\_mom\_work &     \orange{ 0.13} &                   \green{ 0.01} &                      \green{ 0.01} &              \green{ 0.01} &                      \emark \\
kidscore\_mom\_work &  kidiq\_with\_mom\_work &     \green{ 0.03} &                   \green{ 0.01} &                      \green{ 0.01} &              \green{ 0.01} &                      \emark \\
kidscore\_momhs &                kidiq &     \green{ 0.11} &                   \green{ 0.02} &                      \green{ 0.01} &              \green{ 0.01} &                      \emark \\
kidscore\_momhsiq &                kidiq &     \green{ 0.03} &                   \green{ 0.01} &                      \green{ 0.01} &              \green{ 0.01} &                      \emark \\
kidscore\_momiq &                kidiq &     \green{ 0.02} &                   \green{ 0.02} &                      \green{ 0.02} &              \green{ 0.02} &                      \emark \\
kilpisjarvi &      kilpisjarvi\_mod &     \green{ 0.02} &                   \green{ 0.01} &                      \green{ 0.01} &              \green{ 0.01} &                      \emark \\
log10earn\_height &             earnings &     \orange{ 0.56} &                   \green{ 0.00} &                      \green{ 0.00} &              \green{ 0.00} &                      \emark \\
logearn\_height &             earnings &     \green{ 0.10} &                   \green{ 0.00} &                      \green{ 0.00} &              \green{ 0.00} &                      \emark \\
logearn\_height\_male &             earnings &     \green{ 0.07} &                   \green{ 0.01} &                      \green{ 0.01} &              \green{ 0.01} &                      \emark \\
logearn\_interaction &             earnings &     \orange{ 0.20} &                   \green{ 0.01} &                      \green{ 0.01} &              \green{ 0.01} &                      \emark \\
logearn\_interaction\_z &             earnings &     \orange{ 0.14} &                   \green{ 0.02} &                      \green{ 0.01} &              \green{ 0.01} &                      \emark \\
logearn\_logheight\_male &             earnings &     \green{ 0.04} &                   \green{ 0.01} &                      \green{ 0.02} &              \green{ 0.02} &                      \emark \\
logmesquite &             mesquite &     \orange{ 0.20} &                   \green{ 0.01} &                      \green{ 0.01} &              \green{ 0.01} &                      \emark \\
logmesquite\_logva &             mesquite &     \orange{ 0.31} &                   \green{ 0.01} &                      \green{ 0.01} &              \green{ 0.01} &                      \emark \\
logmesquite\_logvas &             mesquite &     \green{ 0.05} &                   \green{ 0.01} &                      \green{ 0.01} &              \green{ 0.01} &                      \emark \\
logmesquite\_logvash &             mesquite &     \orange{ 0.17} &                   \green{ 0.01} &                      \green{ 0.01} &              \green{ 0.01} &                      \emark \\
logmesquite\_logvolume &             mesquite &     \orange{ 0.15} &                   \green{ 0.01} &                      \green{ 0.01} &              \green{ 0.01} &                      \emark \\
lotka\_volterra &     hudson\_lynx\_hare &     \green{ 0.04} &                      \emark &                         \emark &                 \emark &                      \emark \\
low\_dim\_gauss\_mix &    low\_dim\_gauss\_mix &     \orange{ 0.84} &                      --- &                     \orange{ 10.07} &             \orange{ 15.11} &                  \orange{ 30.22} \\
mesquite &             mesquite &     \green{ 0.01} &                   \green{ 0.01} &                      \green{ 0.02} &              \green{ 0.02} &                      \emark \\
    nes &              nes1980 &     \green{ 0.03} &                   \green{ 0.01} &                      \green{ 0.02} &              \green{ 0.02} &                      \emark \\
    nes &              nes1976 &     \green{ 0.03} &                   \green{ 0.01} &                      \green{ 0.01} &              \green{ 0.01} &                      \emark \\
    nes &              nes1992 &     \orange{ 0.06} &                   \green{ 0.01} &                      \green{ 0.01} &              \green{ 0.01} &                      \emark \\
    nes &              nes1984 &     \orange{ 0.08} &                   \green{ 0.01} &                      \green{ 0.02} &              \green{ 0.02} &                      \emark \\
    nes &              nes1972 &     \green{ 0.04} &                   \green{ 0.01} &                      \green{ 0.01} &              \green{ 0.01} &                      \emark \\
    nes &              nes2000 &     \green{ 0.05} &                   \green{ 0.01} &                      \green{ 0.01} &              \green{ 0.01} &                      \emark \\
    nes &              nes1996 &     \green{ 0.04} &                   \green{ 0.01} &                      \green{ 0.01} &              \green{ 0.01} &                      \emark \\
    nes &              nes1988 &     \orange{ 0.13} &                   \green{ 0.01} &                      \green{ 0.01} &              \green{ 0.01} &                      \emark \\
one\_comp\_mm\_elim\_abs & one\_comp\_mm\_elim\_abs &     \green{ 0.02} &                      \emark &                         \emark &                 \emark &                      \emark \\
\bottomrule
\end{tabular}
\end{small}
\end{table*}

\begin{table*}
\caption{Comparing inference speed on PosteriorDB models. For each columns we report the duration in second with the following format mean\gray{(std)}. Results are averaged over 5 runs.}
\label{tab:eval_speed_full}
\begin{small}
      \begin{tabular}{@{}llr@{ }lr@{ }lr@{ }lr@{ }l@{}}

        &&&&\multicolumn{6}{c}{\textsc{NumPyro}}
        \\
        \cmidrule(lr){5-10}
        \textsc{Model} & \textsc{Dataset} & \multicolumn{2}{c}{\textsc{Stan}} & \multicolumn{2}{c}{\textsc{Compr.}} & \multicolumn{2}{c}{\textsc{Mixed}} & \multicolumn{2}{c}{\textsc{Gener.}}
        \\
            \toprule
            accel\_gp &            mcycle\_gp &      1102 &     \gray{(41.6)} &                         &                                &                 &                        &                      &                             \\
            arK &                  arK &        57 &      \gray{(2.7)} &                         38 &                       \gray{(2.0)} &                 37 &               \gray{(0.7)} &                      34 &                    \gray{(0.6)} \\
         arma11 &                 arma &        96 &    \gray{(206.0)} &                         42 &                       \gray{(0.9)} &               1306 &            \gray{(2818.3)} &                    1133 &                 \gray{(2428.4)} \\
            blr &                sblrc &         2 &      \gray{(0.0)} &                          6 &                       \gray{(0.1)} &                  6 &               \gray{(0.7)} &                      &                             \\
            blr &                sblri &         1 &      \gray{(0.0)} &                          6 &                       \gray{(0.1)} &                  6 &               \gray{(0.1)} &                      &                             \\
       diamonds &             diamonds &      2460 &     \gray{(40.7)} &                         &                                &                 &                        &                      &                             \\
           dogs &                 dogs &        66 &      \gray{(2.0)} &                        382 &                       \gray{(6.4)} &                372 &              \gray{(14.1)} &                     369 &                    \gray{(8.8)} \\
       dogs\_log &                 dogs &        32 &      \gray{(2.3)} &                        223 &                       \gray{(9.2)} &                214 &               \gray{(7.0)} &                      &                             \\
    earn\_height &             earnings &        78 &      \gray{(1.1)} &                         15 &                       \gray{(0.1)} &                 15 &               \gray{(0.6)} &                      &                             \\
eight\_schools\_centered &        eight\_schools &         5 &      \gray{(6.2)} &                          7 &                       \gray{(0.5)} &                  6 &               \gray{(0.2)} &                       6 &                    \gray{(0.3)} \\
eight\_schools\_noncentered &        eight\_schools &         1 &      \gray{(0.1)} &                          6 &                       \gray{(0.4)} &                  6 &               \gray{(0.7)} &                       6 &                    \gray{(0.4)} \\
        garch11 &                garch &        17 &      \gray{(0.6)} &                        127 &                       \gray{(2.4)} &                124 &               \gray{(2.9)} &                      &                             \\
   gp\_pois\_regr &         gp\_pois\_regr &        45 &      \gray{(2.9)} &                         &                                &                 &                        &                      &                             \\
        gp\_regr &         gp\_pois\_regr &         2 &      \gray{(0.1)} &                         &                                &                 &                        &                      &                             \\
    hmm\_drive\_0 &  bball\_drive\_event\_0 &       230 &      \gray{(4.7)} &                       1542 &                      \gray{(37.4)} &               1538 &              \gray{(26.7)} &                      &                             \\
    hmm\_drive\_1 &  bball\_drive\_event\_1 &       185 &     \gray{(18.1)} &                        712 &                     \gray{(110.7)} &                750 &              \gray{(86.0)} &                      &                             \\
    hmm\_example &          hmm\_example &        28 &      \gray{(1.8)} &                         62 &                       \gray{(1.2)} &                 62 &               \gray{(1.6)} &                      &                             \\
kidscore\_interaction &                kidiq &       102 &      \gray{(3.4)} &                         13 &                       \gray{(0.4)} &                 13 &               \gray{(0.3)} &                      &                             \\
kidscore\_interaction\_c &  kidiq\_with\_mom\_work &         9 &      \gray{(0.0)} &                          6 &                       \gray{(0.2)} &                  6 &               \gray{(0.2)} &                      &                             \\
kidscore\_interaction\_c2 &  kidiq\_with\_mom\_work &        10 &      \gray{(0.3)} &                          6 &                       \gray{(0.1)} &                  6 &               \gray{(0.1)} &                      &                             \\
kidscore\_interaction\_z &  kidiq\_with\_mom\_work &         9 &      \gray{(0.2)} &                          6 &                       \gray{(0.6)} &                  7 &               \gray{(0.4)} &                      &                             \\
kidscore\_mom\_work &  kidiq\_with\_mom\_work &        14 &      \gray{(0.3)} &                          9 &                       \gray{(0.3)} &                  9 &               \gray{(0.2)} &                      &                             \\
 kidscore\_momhs &                kidiq &         5 &      \gray{(0.1)} &                          6 &                       \gray{(0.1)} &                  6 &               \gray{(0.2)} &                      &                             \\
kidscore\_momhsiq &                kidiq &        28 &      \gray{(0.8)} &                          8 &                       \gray{(0.7)} &                  8 &               \gray{(0.0)} &                      &                             \\
 kidscore\_momiq &                kidiq &        13 &      \gray{(0.2)} &                          7 &                       \gray{(0.2)} &                  7 &               \gray{(0.2)} &                      &                             \\
    kilpisjarvi &      kilpisjarvi\_mod &        59 &      \gray{(0.5)} &                         21 &                       \gray{(0.2)} &                 21 &               \gray{(0.6)} &                      &                             \\
log10earn\_height &             earnings &        81 &      \gray{(1.3)} &                         16 &                       \gray{(0.8)} &                 15 &               \gray{(0.4)} &                      &                             \\
 logearn\_height &             earnings &        79 &      \gray{(2.0)} &                         15 &                       \gray{(0.4)} &                 15 &               \gray{(0.2)} &                      &                             \\
logearn\_height\_male &             earnings &       225 &      \gray{(5.9)} &                         23 &                       \gray{(0.3)} &                 23 &               \gray{(0.6)} &                      &                             \\
logearn\_interaction &             earnings &       563 &     \gray{(14.5)} &                         42 &                       \gray{(1.1)} &                 40 &               \gray{(1.1)} &                      &                             \\
logearn\_interaction\_z &             earnings &        41 &      \gray{(0.8)} &                          8 &                       \gray{(0.2)} &                  9 &               \gray{(0.6)} &                      &                             \\
logearn\_logheight\_male &             earnings &       867 &     \gray{(18.4)} &                         75 &                       \gray{(2.0)} &                 75 &               \gray{(2.5)} &                      &                             \\
    logmesquite &             mesquite &         8 &      \gray{(0.1)} &                          7 &                       \gray{(0.1)} &                  7 &               \gray{(0.1)} &                      &                             \\
logmesquite\_logva &             mesquite &         5 &      \gray{(0.3)} &                          6 &                       \gray{(0.2)} &                  6 &               \gray{(0.2)} &                      &                             \\
logmesquite\_logvas &             mesquite &        14 &      \gray{(0.3)} &                          8 &                       \gray{(0.1)} &                  8 &               \gray{(0.4)} &                      &                             \\
logmesquite\_logvash &             mesquite &        12 &      \gray{(0.2)} &                          7 &                       \gray{(0.1)} &                  7 &               \gray{(0.1)} &                      &                             \\
logmesquite\_logvolume &             mesquite &         1 &      \gray{(0.0)} &                          6 &                       \gray{(0.6)} &                  5 &               \gray{(0.1)} &                      &                             \\
 lotka\_volterra &     hudson\_lynx\_hare &       186 &      \gray{(7.2)} &                         &                                &                 &                        &                      &                             \\
low\_dim\_gauss\_mix &    low\_dim\_gauss\_mix &        58 &      \gray{(0.5)} &                         26 &                       \gray{(0.3)} &                 28 &               \gray{(0.3)} &                      27 &                    \gray{(0.4)} \\
       mesquite &             mesquite &        15 &      \gray{(0.3)} &                          8 &                       \gray{(0.3)} &                  8 &               \gray{(0.1)} &                      &                             \\
            nes &              nes1980 &       216 &      \gray{(5.3)} &                         15 &                       \gray{(0.3)} &                 15 &               \gray{(0.5)} &                      &                             \\
            nes &              nes1976 &       419 &      \gray{(8.0)} &                         20 &                       \gray{(0.4)} &                 21 &               \gray{(0.5)} &                      &                             \\
            nes &              nes1992 &       478 &     \gray{(11.1)} &                         25 &                       \gray{(0.7)} &                 25 &               \gray{(1.1)} &                      &                             \\
            nes &              nes1984 &       433 &     \gray{(10.5)} &                         24 &                       \gray{(0.5)} &                 24 &               \gray{(0.3)} &                      &                             \\
            nes &              nes1972 &       478 &     \gray{(14.2)} &                         25 &                       \gray{(0.9)} &                 25 &               \gray{(0.7)} &                      &                             \\
            nes &              nes2000 &       161 &      \gray{(3.5)} &                         13 &                       \gray{(0.2)} &                 13 &               \gray{(0.2)} &                      &                             \\
            nes &              nes1996 &       398 &     \gray{(10.1)} &                         22 &                       \gray{(0.3)} &                 22 &               \gray{(0.4)} &                      &                             \\
            nes &              nes1988 &       327 &      \gray{(6.1)} &                        120 &                       \gray{(2.0)} &                121 &               \gray{(1.9)} &                      &                             \\
one\_comp\_mm\_elim\_abs & one\_comp\_mm\_elim\_abs &       970 &     \gray{(67.1)} &                         &                                &                 &                        &                      &                             \\
\bottomrule
\end{tabular}
\end{small}
\end{table*}

 \else
\fi

\end{document}